\documentclass{article}
\usepackage{iclr2027_conference,times}
\iclrfinalcopy
\usepackage{hyperref}
\usepackage{url}
\usepackage{amsmath,amssymb,amsthm,bm}
\usepackage{graphicx}
\usepackage{booktabs}
\usepackage{array}
\usepackage{multirow}
\usepackage{xcolor}
\usepackage{colortbl}
\usepackage{tikz}
\usetikzlibrary{arrows.meta,positioning,calc,decorations.pathmorphing,fit,backgrounds}
\usepackage{subcaption}
\usepackage{float}
\usepackage{textcomp}
\usepackage{algorithm}
\usepackage{algorithmic}
\usepackage{wrapfig}

\newcommand{\methodname}{CFRL}
\newcommand{\method}{\textsc{\methodname}}
\newcommand{\methodGRPO}{\method{}-GRR}
\newcommand{\methodSVGScalar}{\method{}-SVG-1D}
\newcommand{\methodSVGTwoD}{\method{}-SVG-2D}

\usepackage{amsmath,amsfonts,bm}

\def\eqref#1{equation~\ref{#1}}

\def\1{\bm{1}}

\def\rvc{{\mathbf{c}}}

\def\rvh{{\mathbf{h}}}

\def\rvw{{\mathbf{w}}}
\def\rvx{{\mathbf{x}}}

\def\rvz{{\mathbf{z}}}

\DeclareMathAlphabet{\mathsfit}{\encodingdefault}{\sfdefault}{m}{sl}
\SetMathAlphabet{\mathsfit}{bold}{\encodingdefault}{\sfdefault}{bx}{n}

\makeatletter
\newcommand{\eqalias}[1]{\begingroup\def\@currentlabel{\theequation}\label{#1}\endgroup\ignorespaces}
\newcommand{\eqaliasprev}[1]{\begingroup\edef\@currentlabel{\the\numexpr\value{equation}-1\relax}\label{#1}\endgroup\ignorespaces}
\makeatother

\newcommand{\hl}[1]{\textcolor{blue!70!black}{#1}}
\definecolor{ctwod}{HTML}{2A6F97}\definecolor{cscalar}{HTML}{E76F51}\definecolor{cgrpo}{HTML}{C9553D}\definecolor{cctrl}{HTML}{495057}
\definecolor{cgpt}{HTML}{2A9D8F}\definecolor{cgrey}{HTML}{8D99AE}\definecolor{cblack}{HTML}{212529}\definecolor{cpurple}{HTML}{6A4C93}

\newcommand{\OURScd}{51.4}\newcommand{\OURScdo}{47.8}\newcommand{\OURScdg}{45.6}\newcommand{\OURScdt}{51.1}
\newcommand{\OURSkk}{99.1}\newcommand{\OURSkko}{95.7}\newcommand{\OURSkkg}{93.8}
\newcommand{\OURScdthree}{50.0}\newcommand{\OURScdthreeo}{46.3}\newcommand{\OURScdthreeg}{43.6}\newcommand{\OURScdthreet}{50.2}\newcommand{\OURSpoolthree}{95.6k / 37.2k}\newcommand{\OURSpoolfour}{99.1k / 37.8k}

\newcommand{\CDa}[1]{\ifcase#1 \or 0.93\or 0.97\or 0.94\or 1.00\or \textbf{58.2}\or \textbf{57.5}\fi}
\newcommand{\CDb}[1]{\ifcase#1 \or 0.53\or 0.52\or 0.59\or 0.59\or 47.0\or 12.9\fi}
\newcommand{\CDc}[1]{\ifcase#1 \or 0.65\or 0.65\or 0.82\or 0.81\or 51.7\or 13.4\fi}
\newcommand{\CDd}[1]{\ifcase#1 \or 0.43\or 0.39\or 0.67\or 0.68\or ---\or ---\fi}
\newcommand{\GBa}[1]{\ifcase#1 \or 25.4\or 2{,}783\or 4{,}978\fi}
\newcommand{\GBb}[1]{\ifcase#1 \or 25.2\or 2{,}748\or 4{,}960\fi}
\newcommand{\GBc}[1]{\ifcase#1 \or 52.8\or 3{,}712\or 10{,}204\fi}
\newcommand{\GBd}[1]{\ifcase#1 \or 28.4\or 3{,}069\or 5{,}647\fi}
\newcommand{\GBe}[1]{\ifcase#1 \or \textbf{59.3}\or \textbf{3{,}933}\or \textbf{11{,}753}\fi}
\newcommand{\ABp}[1]{\ifcase#1 \or 49.9\or 54.8\or 53.0\or 51.7\or 51.8\or 51.9\or 54.8\fi}
\newcommand{\ABm}[1]{\ifcase#1 \or 45.5\or 52.2\or 54.2\or 55.3\or 55.7\or 55.5\or 55.7\fi}
\newcommand{\ABcp}[1]{\ifcase#1 \or 44.5\or 45.6\or 47.6\or 48.2\or 47.5\or 48.0\or 48.2\fi}
\newtheorem{proposition}{Proposition}

\title{Backpropagating Verifier Rewards Through Continuous Flow Language Models}
\author{
\textbf{Xu Han}\thanks{Equal contribution.}\\
\normalfont Tufts University
\And
\textbf{Jiajing Hu}\footnotemark[1]\\
\normalfont Tufts University
\And
\textbf{Li-Ping Liu}\thanks{Contact: \texttt{liping.liu@tufts.edu}.}\\
\normalfont Tufts University
}

\begin{document}
\maketitle
\lhead{Preprint. Under review.}
\begin{abstract}
Reinforcement learning (RL) has substantially improved the reasoning of autoregressive and discrete-diffusion language models. Language models that generate continuous token embeddings are emerging as a strong alternative, but RL post-training for them remains understudied. We propose \method{}, a post-training framework for continuous flow language models. Such a model denoises token embeddings over multiple steps and decodes them into discrete tokens only at the last step. To obtain pathwise gradients, which backpropagate through continuous actions and have much lower variance than REINFORCE, we introduce a critic that estimates the value of intermediate token embeddings, bypassing the non-differentiable decoding step. The critic must also keep track of the value function as the policy changes; to this end, we design a critic that predicts a value from every token representation and improves policy optimization. Across four reasoning tasks, \method{} reliably improves a continuous flow language model through RL post-training and outperforms baseline methods of the same type. Our code is available at \url{https://github.com/tufts-ml/CFRL}.
\end{abstract}

\section{Introduction}
\label{sec:intro}
Reinforcement learning (RL) post-training has become a central mechanism for turning pretrained autoregressive (AR) language 
models into reasoners \citep{ouyang2022instructgpt,shao2024deepseekmath,deepseek2025r1}. The explicit token probabilities of AR models
provide a natural RL interface: each token choice is an action that an advantage can reinforce or suppress.
However, AR generation requires one sequential model evaluation per token, tying latency to sequence length. Diffusion-based
language models ease this bottleneck by jointly refining all token positions over a fixed number of denoising steps,
which can substantially speed up generation \citep{nie2025llada,sahoo2024mdlm,lou2023sedd}.
RL has recently been applied to diffusion-based language models
with promising results \citep{zhao2025d1,zhu2025llada15,shao2026papo}. Because these models operate on
discrete tokens and take discrete actions during generation, RL methods for them, including PPO \citep{schulman2017ppo}
and GRPO \citep{shao2024deepseekmath}, rely mainly on score-function (REINFORCE) gradients \citep{williams1992reinforce}.

Continuous language models (CLMs) have historically lagged behind discrete counterparts in language generation 
\citep{li2022diffusionlm,gong2022diffuseq,dieleman2022continuous}. The recent ELF \citep{hu2026elf}, a CLM based on
flow matching \citep{lipman2022flow,liu2022rectified}, closes this gap: it achieves lower generative perplexity with fewer
sampling steps while using ten times fewer training tokens. ELF progressively denoises Gaussian noise into real-valued token embeddings, then decodes the final embeddings into discrete tokens. Its continuous generation trajectory is
differentiable, which makes it possible to backpropagate reward signals through the denoising process.


In this work, we propose \textbf{C}ontinuous \textbf{F}low \textbf{R}einforcement learning for \textbf{L}anguage models (\method{}), 
a post-training framework for improving CLM reasoning. We first adapt group-relative REINFORCE \citep{yao2026grouprelative}
to CLM post-training, using a verifier applied to the decoded response as the reward. We then exploit the differentiability
of the continuous trajectory. Because the CLM's decoder is not differentiable, we train a critic to predict the expected verifier
outcome from intermediate denoising states. Near the end of the trajectory,
the critic can predict this outcome accurately because the state contains nearly all the information needed to determine
the decoded response. Its prediction is a differentiable reward estimate that can be backpropagated through a flow
transition. We further design a novel critic architecture that propagates reward feedback to each token representation:
it predicts a value from every token representation and trains each prediction against the sequence-level verifier outcome.
This design turns the verifier signal into position-specific gradients, enabling more effective policy optimization.

We evaluate our framework under a unified protocol on four program-verifiable reasoning tasks: Countdown, Sudoku, Knights-and-Knaves, and Game24. 
Our models perform comparably to or better than baselines of similar or larger size. Our contributions are as follows.
\begin{itemize}
    \item To the best of our knowledge, we demonstrate the first reliable improvement in the reasoning capabilities of a CLM through pathwise reinforcement learning.
    \item We propose a novel critic architecture that consistently improves RL performance and can be readily adapted to other denoising-based generative models.
    \item We evaluate our approach on four program-verifiable reasoning tasks and demonstrate its effectiveness.
\end{itemize}

\section{Related Work}
\label{sec:related}

\paragraph{Diffusion and flow language models.}
Discrete diffusion language models, including SEDD, MDLM, and LLaDA \citep{lou2023sedd,sahoo2024mdlm,nie2025llada}, generate through discrete or masked-token trajectories. Continuous-state language generation instead diffuses token embeddings, beginning with Diffusion-LM \citep{li2022diffusionlm}. Although earlier continuous models lagged behind discrete diffusion, ELF \citep{hu2026elf} closes this gap and outperforms discrete counterparts by using flow matching \citep{lipman2022flow,liu2022rectified,albergo2023stochastic} in a continuous embedding space and decoding tokens only at the trajectory endpoint.

\paragraph{RL for denoising generative models.}
RL for image diffusion commonly casts denoising as a sequential decision process and applies score-function policy gradients, as in DDPO and DPOK \citep{black2023ddpo,fan2023dpok}. Flow-GRPO, DanceGRPO, and MixGRPO adapt this approach to flow-matching models \citep{liu2025flowgrpo,xue2025dancegrpo,li2025mixgrpo}, while FPO uses the flow-matching objective as a likelihood surrogate within PPO \citep{mcallister2025fpo}. RL for discrete diffusion language models relies on score-function updates using one-step likelihood approximations, variational objectives, or actions at intermediate denoising states \citep{zhao2025d1,zhu2025llada15,shao2026papo}. 

\paragraph{RL for autoregressive language models.}
RL post-training has enabled autoregressive language models, including OpenAI-o1 \citep{openai2024o1} and DeepSeek-R1 \citep{deepseek2025r1}, to become strong reasoners. These methods optimize categorical token actions with score-function methods such as PPO and GRPO \citep{shao2024deepseekmath,schulman2017ppo}, with recent variants modifying clipping and sampling \citep{yu2025dapo,zheng2025gspo}.

\section{Preliminaries: Continuous Flow Language Modeling and RL}
\label{sec:prelim}

\paragraph{Continuous flow language modeling.}
We build our work on ELF \citep{hu2026elf}, a flow-based language model. ELF represents text tokens using embeddings produced by a frozen T5 encoder \citep{raffel2020t5}. Let $\rvx=(x_1,\ldots,x_L)$ denote the embeddings of the tokens in a target sequence $\rvw=(w_1,\ldots,w_L)$. When the target sequence is conditioned on additional context, such as a prompt, we denote the corresponding context-token embeddings by $\rvc$. 

ELF models the conditional distribution $p_\theta(\rvx \mid \rvc)$ using flow matching (FM), where a neural network learns to transform noisy states into clean token embeddings. At time $t\in[0,1]$, the noisy state is defined as $\rvz_t=(1-t)\,\epsilon+t\,\rvx$, where $\epsilon\sim\mathcal N(0, c_\sigma^2I)$ is Gaussian noise with fixed scale $c_\sigma^2$, and $\rvx$ denotes the clean target embeddings. ELF adopts an $x$-prediction parameterization \citep{li2026back, han2026self}: given a noisy state $\rvz_t$, the denoiser predicts an estimate $\rvx_t$ of the clean embeddings $\rvx$. Let $0=t_0<\cdots<t_K=1$, denote a discretization of the flow time. The corresponding denoising update is
\begin{align} 
\rvx_{t_{k}} = \mathrm{d}_\theta(\rvz_{t_{k}}, t_{k}, \rvx_{t_{k-1}}, \rvc)
\label{eq:denoiser}
\end{align} 
We explicitly incorporate the self-conditioning \citep{chen2022analog}, whereby the denoiser $\mathrm{d}_\theta$ takes the previous clean-embedding estimate $\rvx_{t_{k-1}}$ as an additional input. The denoiser $\mathrm{d}_\theta$ is implemented as a diffusion Transformer \citep{peebles2023scalable}, with $\theta$ denoting its learnable parameters. It is trained by minimizing the mean squared error (MSE) between the predicted clean embedding $\rvx_t$ and the corresponding ground-truth embedding $\rvx$.

At inference time, we consider the SDE sampling of ELF, which injects stochastic noise at each generation step. This stochastic sampling is particularly useful for the reinforcement learning framework introduced later, as it induces a distribution over generation trajectories rather than a single deterministic path.
\begin{align} 
\rvz_{t_k} =  \lambda_k  \rvz_{t_{k-1}} + \eta_k \rvx_{t_{k-1}} + \psi_k \epsilon, \qquad k=1, \ldots, K  
\label{eq:sde}
\end{align} 
The scalars $\lambda_k$, $\eta_k$, and $\psi_k$ are determined by the step size and the noise reinjection scale specified by the ELF sampling algorithm. Sampling is initialized with $\rvz_0 \sim \mathcal{N}(0,c_\sigma^2 I)$ and $\rvx_{t_{-1}}=\mathbf{0}$. After $K$ denoising steps, the procedure produces $\rvz_1$, which corresponds  to the clean embedding $\rvx$.

ELF further learns a decoder $p_\theta(\rvw \mid \rvx)$ that maps the final denoised embeddings $\rvz_1$ to the discrete token sequence $\rvw$. Specifically, each embedding in $\rvz_1$ is projected to the token vocabulary through a linear output layer, yielding a distribution over discrete tokens.
\begin{align}
w_i \sim \mathrm{softmax}(W \rvz_{1}[i]), \qquad i=1, \ldots, L
\end{align}
Here, $W$ denotes the linear output layer that maps each embedding to a vector of vocabulary logits. We include $W$ in the set of learnable parameters $\theta$. Together, $p_\theta(\rvx\mid\rvc)$ and the decoder $p_\theta(\rvw\mid\rvx)$ define the conditional language distribution $p_\theta(\rvw\mid\rvc)$. A derivation of the stochastic ELF sampler is provided in Appendix~\ref{app:kernels}.

\section{Method}
The generation procedure of ELF introduced in Section~\ref{sec:prelim} proceeds over a fixed horizon of $K+1$ steps. At each generation step $k\in\{1,\ldots,K\}$, we define the MDP state as $s_k = (\rvz_{t_k}, \rvx_{t_{k-1}})$. The initial state is $s_0 = (\epsilon, \mathbf{0})$, where $\epsilon$ is the Gaussian noise defined above. We include $\rvx_{t_{k-1}}$ in the state because it is required by the self-conditioning mechanism of ELF. We omit $\rvc$ and $t_{k}$ from the state $s_k$ for notational simplicity. The terminal state $s_{K+1} = \rvw$  corresponds to the decoded response. 

The actions consist of the denoising predictions followed by the final decoding step. Specifically, for $k \in \{0, \ldots, K-1\}$, we define the action as $a_{k} = (\eta_k \rvx_{t_{k}} + \psi_k \epsilon)$. The corresponding policy $\pi_\theta(a_k | s_k)$ defined by the denoiser $\mathrm{d}_\theta$ and the noise distribution. The final action is $a_{K} = \rvw$, and the policy at this step is $\pi_\theta(a_K | s_K)=p_\theta(\rvw | \rvz_{t_K})$. 

The transition probability is defined by the denoising process in~(\ref{eq:sde}). A terminal reward  $R(s_{K+1})$ is provided by a verifier that evaluates 
the decoded response $\rvw$. The reinforcement learning objective is to maximize the expected terminal reward over trajectories induced by the policy:
\begin{equation}
\max_\theta \mathbb E_{s_0\sim p_0, a_k\sim \pi_\theta(\cdot|s_k), s_{k+1}\sim p(\cdot|s_k, a_k)}[R(s_{K+1})].
\label{eq:objective}
\end{equation}

\subsection{CFRL-GRR: Learning with group-relative REINFORCE}
\label{sec:grr}

Similar to other language models, an ELF model can be post-trained using policy-gradient methods \citep{sutton1998reinforcement}, without concerning the last non-differentible decoding step. We start with a simple method and consider a variant of the group-relative REINFORCE method \citep{yao2026grouprelative}. 
For each context $\rvc$, we sample a group of $G$ responses, $(s^1_{K+1}, \ldots, s^G_{K+1})$. For each response, we additionally record the corresponding generation trajectory. Let 
$r^g = R(s^g_{K+1})$, then the group advantage is computed as $A^g =  \frac{r(s^g_{K+1}) - \mathrm{mean}(r^{1:G})}{\mathrm{std}(r^{1:G})}$.
The maximization objective is:
\begin{equation}
\mathcal J_{\mathrm{GRR}}(\theta)=\frac1G\sum_{g=1}^{G}A^g\Bigl[\sum_{k=0}^{K-1}\log\pi_\theta\bigl(a^g_k \mid s^g_{k}\bigr)+\frac{1}{L} \log\pi_\theta\bigl(a^g_{K} \mid s^g_K\bigr)\Bigr].
\label{eq:grpo_loss}
\end{equation}
The last step, $\pi_\theta\bigl(a^g_{K} \mid s^g_K\bigr)$, decodes all $L$ token embeddings with a shared output matrix.
Because this decoding step is comparatively easy to optimize, we scale its term by $\frac{1}{L}$ to slow its updates.

The key to the optimization of $\mathcal J_{\mathrm{GRR}}(\theta)$ is to compute the gradient $\nabla_\theta \log\pi_\theta\bigl(a^g_k \mid s^g_{k}\bigr)$ for each group $g$. Below we omit $g$ for notation simplicity. Note that $a_k$ is from a Gaussian distribution $\mathcal{N}(\eta_k \rvx_{t_{k}}, \psi_k^2 c_\sigma^2 I)$, and we have:
\begin{align*}
\nabla_\theta \log\pi_\theta\bigl(a_k \mid s_{k}\bigr) &= \frac{\eta_k(a_k - \eta_k \rvx_{t_{k}})^\top }{\psi_{k}^2 c_\sigma^2} \left( \frac{\partial \rvx_{t_{k}}}{\partial \theta} +
\frac{\partial \rvx_{t_{k}}}{\partial \rvx_{t_{k-1}}}\,D_{\theta} \rvx_{t_{k-1}} \right) \approx \frac{\eta_k(a_k - \eta_k \rvx_{t_{k}})^\top }{\psi_{k}^2 c_\sigma^2}\frac{\partial \rvx_{t_{k}}}{\partial \theta}.
\end{align*}
There is a deterministic computation path from $\rvx_{t_0}$ to $\rvx_{t_{k}}$ through the self-conditioning mechanism in the denoiser. The term $D_\theta \rvx_{t_{k-1}}$ is the total derivative of $\rvx_{t_{k-1}}$ with respect to $\theta$ through the path. The exact gradient calculation requires backpropagation through denoising steps, which has the same issue as backpropagation through time. In our implementation, we truncate the backpropagation through steps and only use the local gradient $\frac{\partial \rvx_{t_{k}}}{\partial \theta}$ as the approximate gradient. The derivation details are given in Appendix~\ref{app:variance}. The gradient $\nabla_\theta\log\pi_\theta\bigl(a^g_{K} \mid s^g_K\bigr)$ is from the final decoding step, which is a standard cross-entropy loss and can be computed with the standard backpropagation.

We also experimented with GRPO \citep{shao2024deepseekmath} with various settings and found that neither ratio clipping nor a KL penalty was necessary in our setting.
This observation that the KL-penalty has no benefit is consistent with findings reported in DAPO \citep{yu2025dapo}. One important distinction is that our denoiser operates over a continuous state space of dimension $dL$, whereas an AR language model defines a discrete distribution over a vocabulary at each generation step. These differences in the underlying generation process and policy parameterization may call for different RL optimization strategies. We leave an  investigation of such training techniques to future work.

\subsection{\methodSVGScalar{}: Learn with pathwise RL with a critic}
\label{sec:method}
\begin{wrapfigure}[19]{r}{0.49\columnwidth}
  \vspace{-0.8\baselineskip}
  \centering
  \includegraphics[width=\linewidth]{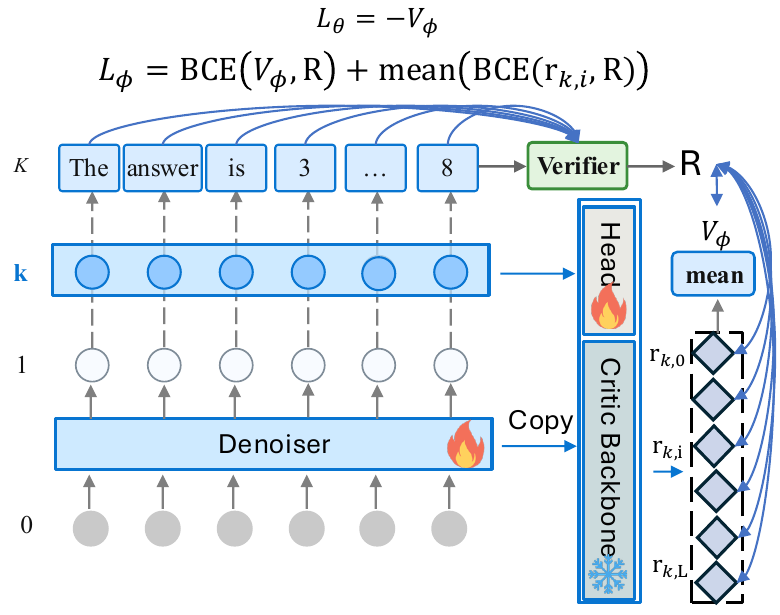}
  \caption{\textbf{Joint actor--critic training.}
  The denoiser serves as the actor. The critic uses a frozen copy of the
  denoiser backbone and a value head trained jointly with the actor.}
  \label{fig:mechanism}
  \vspace{-0.6\baselineskip}
\end{wrapfigure}

In our MDP formulation, only the last decoding step blocks gradient flow from the reward to the policy. We use stochastic 
value gradients (SVG) \citep{heess2015svg} and learn a critic to backpropagate the reward signal in denoising steps.  

The value function $V^\star_k(s_k) = \mathbb E[R(s_{K+1}) \mid s_k]$ of a state $s_k$ is the expected reward of all continuations from that state. 
Since the reward $R(s_{K+1})$ is binary, the value $V^\star_k(s_k)$ is the probability $p_\theta(R(s_{K+1}) | s_k)$ of success from $s_k$. 
Note that the MDP problem here has a finite horizon so the value function is relatively easy to learn. 

We train a critic $V_\phi$ to predict the value function $V^\star_k(s_k)$ from the state $s_k$. The critic is trained with 
a binary cross-entropy loss against the verifier reward $R(s_{K+1})$ at the end of the trajectory. The minimization objective is:
\begin{equation}
\mathcal L_{\mathrm{critic}}(\phi) = \mathbb{E}_{s_k, s_{K+1}} \Bigl[ \mathrm{BCE}(V_\phi(s_k), R(s_{K+1}))\Bigr].
\label{eq:critic_loss}
\end{equation}
Here $\mathrm{BCE}(V_\phi(s_k), R(s_{K+1}))$ is the binary cross-entropy loss of $V_\phi(s_k)$ with the true binary label $R(s_{K+1})$. 
The loss is consistent here: the true probability $p_\theta(R(s_{K+1}) | s_k)$ receives the lowest expected loss.  

The critic can learn the value function fairly accurately. Since the FM model is trained to generate accurate token embeddings that 
can be decoded without much ambiguity. We further freeze the decoder so that denoised token embeddings keep the same semantic. 
The critic can read out the information about the response from these embeddings at the later denoising stage. 

Now we can directly backpropagate the reward information to denoising actions. We don't have reward in non-terminal states, so the maximization objective is
\begin{align}
 \mathcal{J}_{\mathrm{SVG}}(\theta) &= \mathbb{E}_{s_{k - 1}, s_k} \Bigl[ V_\phi(s_k) \Bigr]\label{eq:svg_objective}
= \mathbb{E}_{\rvz_{t_{k - 1}}, \rvx_{t_{k-2}}, \varepsilon} \Bigl[ V_\phi(\rvz_{t_k}, \rvx_{t_{k-1}}) \Bigr] \quad \text{with} \\
\rvx_{t_{k-1}} &= \mathrm{d}_\theta(\rvz_{t_{k-1}}, \rvx_{t_{k-2}}),  \quad \rvz_{t_{k}} = \lambda_{k} \rvz_{t_{k-1}} + \eta_k \rvx_{t_{k-1}} + \psi_k \varepsilon
\end{align}
In the implementation of Equation~\ref{eq:svg_objective}, we maximize the expectation of $\mathrm{logit}(V_\phi(s_k))$ instead of the expected value. This objective will push the actor to stay away from bad solutions; see Appendix~\ref{app:logit_surrogate} for details.
In the calculation of $\nabla_\theta \mathcal{J}_{\mathrm{SVG}}(\theta)$, we only use the gradient from one transition and do not take the derivative of 
state $s_{k-1} = (\rvz_{t_{k - 1}}, \rvx_{t_{k-2}})$. Then the calcualtion of $\mathcal{J}_{\mathrm{SVG}}(\theta)$ can be viewed reparameterization, 
which allows the gradient to propagate from $V_\phi$ to $d_\theta$ directly.    

In our architecture design, the critic $V_\phi$ takes Transformer architecture from the trained EFL model but keeps its own 
copy. The Transformer takes 
$(\rvz_{t_{k}}, t_k, \rvx_{t_{k-1}}, \rvc)$ as input and produce a sequence of  
$L$ token-level hidden representations, $\rvh = (h_1, \ldots, h_L)$. Then mean pooling is applied over 
these hidden representations to give a sequence-level representation, which is then mapped to a scalar value prediction:
\begin{align}
\label{eq:value1}
V_{\phi}(s_k) = \mathrm{sigmoid}(f_{\phi}(\mathrm{mean}(\rvh)))    
\end{align}
Here $f_{\phi}$ is an MLP, whose parameters are included in $\phi$.

\begin{figure}[!t]
\centering
\includegraphics[width=\textwidth]{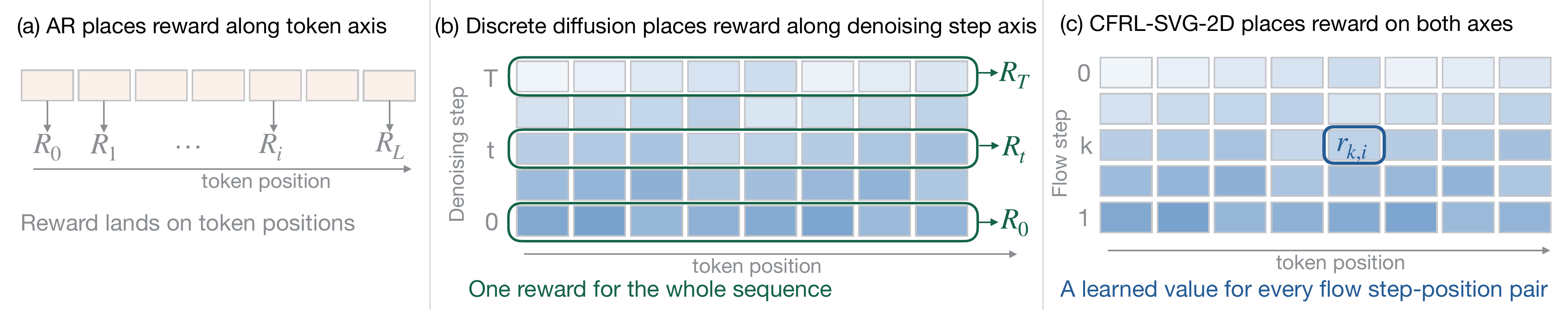}
\caption{Where the reward lands: (a) on the token axis in an AR model; (b) on the step axis in RL for diffusion and flow models; (c) on both in \methodSVGTwoD{}: $r_{k,i}$ for flow step $k$ and position $i$, each an estimate of the whole response's outcome.}
\label{fig:2dcritic}
\end{figure}


\subsection{\methodSVGTwoD{}: Token-wise reward feedback}
\label{sec:critic}\label{sec:jpc_scalar}\label{sec:jpc_2d}

A plain design of the critic architecture needs to aggregate transformer outputs $\rvh=(h_1, \ldots, h_L)$ into a single vector and then compute the value from it. The aggregation can be mean pooling (our method) or weighted average with attention weights from a special class token.  
In either case, the aggregation creates a bottleneck for information flow. From the perspective of gradient calculation, the signal from the reward is distributed to multiple tokens. In this subsection, we design a special architecture that allows 
reward signal to directly pass to single tokens. 

In the new neural architecture, we compute a prediction of the reward $r_{k,i}$ for each token and aggregate these token-wise
predictions to get the critic output.  
\begin{align}
\label{eq:value2}
V_\phi(s_k) = \mathrm{sigmoid}\left(\frac{1}{L}\sum_{i=1}^{L} f_{\phi}(h_i)\right), \quad r_{k,i} = \mathrm{sigmoid}(f_{\phi}(h_i)).
\end{align}
Here $h_i$ is the output embedding for token $i$ when the transformer takes $s_k$ as the input. The function $f_{\phi}(\cdot)$ is an MLP head. 

During model training, we compare both the token prediction $r_{k, i}$ and the final prediction $V_\phi(s_k)$ against the reward 
$R(s_k)$. The loss for fitting the new critic is: 
\begin{equation}
\mathcal L_{\rm 2D}(\phi)=\mathbb E_{s_k,s_{K+1}}\Big[\,\mathrm{BCE}\big(V_\phi(s_k), R(s_{K+1}) \big)
+\tfrac1{L}\textstyle\sum_{i=1}^L\mathrm{BCE}\big(r_{k,i}, R(s_{k+1})\big)\Big]. 
\label{eq:critic2d_loss}
\end{equation}\eqalias{eq:bce}
In this new design, each token vector receives the reward signal directly.


\subsection{Architecture and model training}

For each task, we first perform supervised fine-tuning (SFT) on the pre-trained ELF model to obtain a task-specific policy. We then pre-train the critic on this policy's own rollouts: we generate full trajectories with the task-SFT sampler, evaluate each decoded response with the task verifier to obtain its terminal reward $R(s_{K+1})$, and pair a subset of intermediate states $s_k$ of every trajectory with that reward. The critic starts from a copy of the task-SFT denoiser backbone with an added value head $f_\phi$, and backbone and head are both trained on these state--reward pairs with the BCE loss~(\ref{eq:critic_loss}) or~(\ref{eq:critic2d_loss}).

We then train the policy and the critic jointly, alternating one policy update with a few critic updates (Algorithm~\ref{alg:jpc}); the critic backbone stays frozen and only its head is updated. A policy update picks one denoising step from the middle of the trajectory, where the state already carries information about the outcome but the generation has not yet committed to it. We roll out the current policy up to that step without gradients, then differentiate through the single denoiser call at that step and update $\theta$ to increase $V_\phi$: the policy learns to move the trajectory toward states of higher predicted reward. A critic update uses fresh rollouts of the updated policy, decoded and verified, and refits the head to their states and outcomes.

After RL, we add an optional self-training loop. Each round samples the RL policy on unlabelled problems, keeps the answers the verifier accepts, and fine-tunes the model on them with ELF's supervised loss. Since that loss trains the whole denoiser, the gain becomes a property of the model rather than of the sampler used in RL, so faster samplers benefit too.

\begin{algorithm}[t]
\caption{Joint training of the policy and critic.}
\label{alg:jpc}
\small
\begin{algorithmic}[1]
\REQUIRE policy parameters $\theta$; pre-trained critic $V_\phi$ with trainable head parameters $\phi$ on a fixed copy of the denoiser; verifier $R$; sampler $(\gamma,K)$; window $\mathcal W$ with top step $k_{\rm hi}$; buffer $\mathcal B$ and its size $N_{\mathcal B}$; iterations $T$, prompts per iteration $B$, critic minibatch $B'$ and steps $m_{\rm c}$, AUC check every $n_q$ iterations on the newest $N_q$ states with floor $a_0$ and $m'_{\rm c}$ extra steps; learning rates $\eta_\theta,\eta_\phi$
\FOR{iteration $=1,\dots,T$}
  \STATE \textbf{Actor.} For $B$ prompts and one $k\sim\mathcal U(\mathcal W)$: roll out under $(\gamma,K)$ to $s_{k-1}$ without gradients, take the denoiser call at $s_{k-1}$ with gradients to form $s_k$, and set $\theta\leftarrow\theta+\eta_\theta\,\mathrm{AdamW}\big(\nabla_\theta\frac1B\sum_{g=1}^{B}V_\phi(s^g_k)\big)$ \hfill (eq.~\ref{eq:value1}; eq.~\ref{eq:value2}; gradient-norm clipped)
  \STATE \textbf{Buffer.} Roll out $B$ fresh prompts under $(\gamma,K)$ to the end, decode, verify; push each $(s_{k_{\rm hi}},R(S_{K+1}))$ into $\mathcal B$, keeping the newest $N_{\mathcal B}$
  \STATE \textbf{Critic.} $m_{\rm c}$ steps $\phi\leftarrow\phi-\eta_\phi\,\mathrm{AdamW}(\nabla_\phi\mathcal L_{\rm 2D})$ on minibatches of $B'$ states of $\mathcal B$; every $n_q$ iterations, $m'_{\rm c}$ more if the AUC of $V_\phi$ on the newest $N_q$ states is below $a_0$ \hfill (eq.~\ref{eq:critic2d_loss}; eq.~(\ref{eq:critic_loss}))
\ENDFOR
\end{algorithmic}
\end{algorithm}


\section{Experiments}
\label{sec:experiments}
We evaluate \method{} on four reasoning tasks with program verifiers. \textbf{Countdown}: combine four or five given integers with arithmetic operations to reach a target. \textbf{Sudoku}: complete a $9\times9$ grid from its givens. \textbf{Knights-and-Knaves} (K\&K): from the statements of four to six characters who each always tell the truth or always lie, decide who is which. \textbf{Game24}: reach 24 from four given numbers, each used once.

\paragraph{Setup.}
Every model is first fine-tuned on each task (task SFT); the ELF-B (105M) task-SFT model is the base of all flow runs. Held-out prompts are split into a validation half, on which checkpoints are selected, and a test half, on which every accuracy is reported. The deployment sampler is $\gamma56$@64 (churn $\gamma{=}56$, $K{=}64$ steps; $\gamma28$@32 on Sudoku) and the cheap sampler $\gamma4$@16, both through the same denoiser. As references from other architectures we train GPT-2 \citep{radford2019gpt2} with token-level GRPO/PPO and LLaDA-8B-Instruct \citep{nie2025llada} with d1's diffu-GRPO \citep{zhao2025d1} on the same data. More details are in Appendices~\ref{app:datasets} and~\ref{app:details}.

\begin{table}[t]
\centering\small
\begin{minipage}[c]{0.615\linewidth}
\centering\footnotesize
\setlength{\tabcolsep}{1.5pt}
\begin{tabular}{@{}lcccc@{}}
\toprule
\textbf{Model} & \textbf{Countdown} & \textbf{Sudoku} & \textbf{K\&K} & \textbf{Game24} \\
\midrule
\multicolumn{5}{c}{\emph{GPT-2 (124M, autoregressive)}}\\
released model & \phantom{0}0.0 & \phantom{0}0.0 & \phantom{0}0.0 & \phantom{0}0.0 \\
+ task SFT & 45.1 & \phantom{0}0.8 & 92.7 & 28.0 \\
+ token GRPO/PPO & \cellcolor{blue!12}60.9 {\scriptsize(+15.8)} & \phantom{0}0.8 {\scriptsize(+0.0)} & \cellcolor{blue!12}92.8 {\scriptsize(+0.1)} & 46.7 {\scriptsize(+18.7)} \\
\midrule
\multicolumn{5}{c}{\emph{LLaDA-8B-Instruct (masked diffusion)}}\\
released model & \phantom{0}1.4 & \phantom{0}0.0 & \phantom{0}5.4 & \phantom{0}1.6 \\
+ task SFT & 12.1 & \phantom{0}0.5 & 61.4 & 15.9 \\
+ diffu-GRPO & 14.9 {\scriptsize(+2.8)} & \phantom{0}0.8 {\scriptsize(+0.3)} & 63.2 {\scriptsize(+1.8)} & 18.8 {\scriptsize(+2.9)} \\
\midrule
\multicolumn{5}{c}{\emph{ELF-B (105M, continuous flow)}}\\
released model & \phantom{0}0.0 & \phantom{0}0.0 & \phantom{0}0.0 & \phantom{0}0.0 \\
+ task SFT (base) & 42.8 & 64.0 & 68.2 & 44.0 \\
+ \methodGRPO{} & 42.6 {\scriptsize($-$0.2)} & 67.9 {\scriptsize(+3.9)} & 71.4 {\scriptsize(+3.2)} & 47.4 {\scriptsize(+3.4)} \\
+ \methodSVGScalar{} & 47.4 {\scriptsize(+4.6)} & 64.1 {\scriptsize(+0.1)} & 76.5 {\scriptsize(+8.3)} & 51.2 {\scriptsize(+7.2)} \\
\rowcolor{gray!15}\textbf{+ \methodSVGTwoD{}} & \textbf{49.7} {\scriptsize(+6.9)} & \cellcolor{blue!12}\textbf{68.7} {\scriptsize(+4.7)} & \textbf{77.5} {\scriptsize(+9.3)} & \cellcolor{blue!12}\textbf{58.2} {\scriptsize(+14.2)} \\
\bottomrule
\end{tabular}
\end{minipage}\hfill
\begin{minipage}[c]{0.37\linewidth}
\centering
\includegraphics[width=\linewidth]{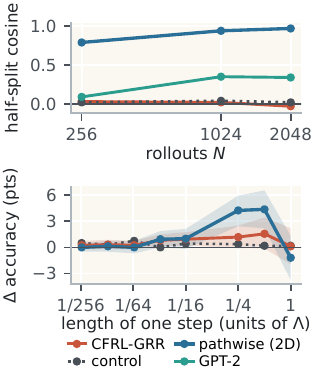}
\end{minipage}
\caption{\textbf{Main results.} \textbf{Left:} accuracy (\%) under the deployment sampler, with the gain over task SFT in parentheses; shaded: recommended; bold: best flow row; blue: best per task. \textbf{Right:} gradient estimates on Game24 base rollouts: half-split cosine (top) and accuracy after one step (bottom).}
\label{tab:main}
\end{table}

\begin{table}[t]
\centering\small
\setlength{\tabcolsep}{4pt}
\begin{tabular*}{\linewidth}{@{\extracolsep{\fill}}lccccc@{}}
\toprule
\textbf{Task} & \textbf{SFT base} & \textbf{\methodGRPO{}} & \textbf{\methodSVGScalar{}} & \textbf{\methodSVGTwoD{}}  \\
\midrule
Countdown & 60.5 & 63.3 {\scriptsize(+2.8)} & 72.6 {\scriptsize(+12.1)} & \textbf{74.4} {\scriptsize(+13.9)}  \\
Sudoku & 71.0 & 70.3 {\scriptsize($-$0.7)} & \textbf{79.3} {\scriptsize(+8.3)} & 77.9 {\scriptsize(+6.9)} \\
Game24 & 58.6 & 66.9 {\scriptsize(+8.3)} & 70.5 {\scriptsize(+11.9)} & \textbf{71.0} {\scriptsize(+12.4)} \\
K\&K & 100.0 & -- & -- & -- \\
\bottomrule
\end{tabular*}
\arrayrulecolor{black}
\caption{\textbf{Results on ELF-L (652M).} Test accuracy (\%). Bold: best flow row.} 
\label{tab:elf_l}
\end{table}

\begin{figure}[t]
\centering
\includegraphics[width=0.24\linewidth]{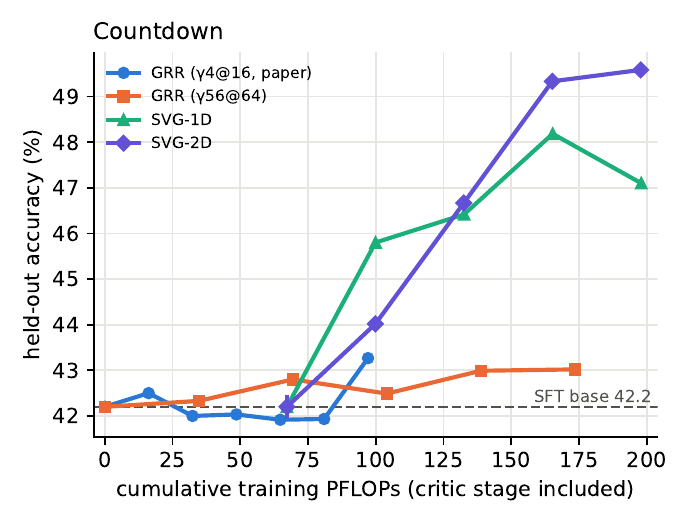}\hfill
\includegraphics[width=0.24\linewidth]{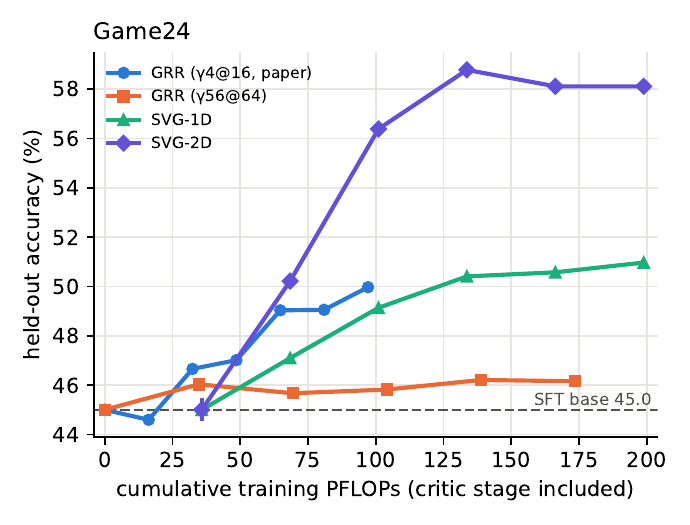}\hfill
\includegraphics[width=0.24\linewidth]{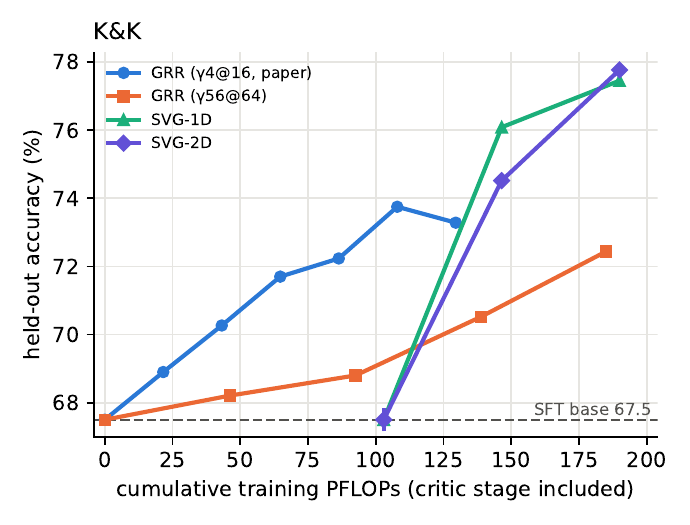}\hfill
\includegraphics[width=0.24\linewidth]{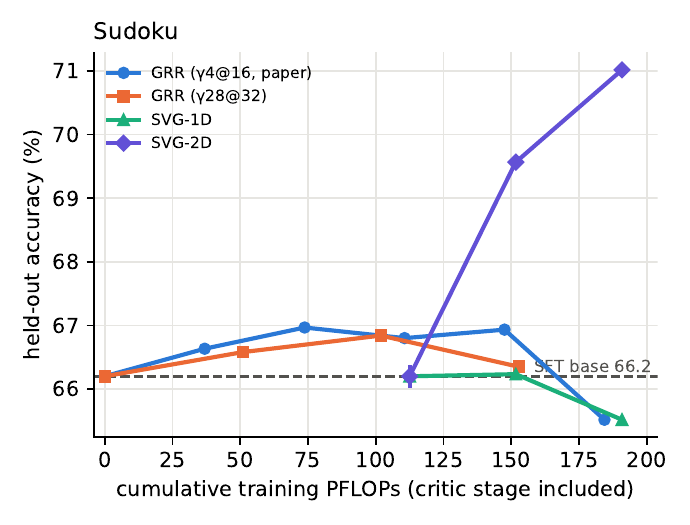}
\caption{Accuracy vs FLOPs across CFRL variants on four datasets.}
\label{fig:accuracy_flops}
\end{figure}

\paragraph{Main results.}\label{sec:exp_main}
Table~\ref{tab:main} compares the three \method{} variants with the two references. Every variant raises the accuracy of the base model over the four tasks; \methodGRPO{}, which needs no critic, provides the least improvement and \methodSVGTwoD{} is the best flow row on every task. Across architectures, GPT-2 leads on Countdown and K\&K, which require longer derivations, whereas the flow model leads on Game24 and Sudoku, the latter of which neither reference model learns. The 8B masked-diffusion model stays far below the flow model on every task, and \methodSVGTwoD{} gains more than diffu-GRPO on every task. Two patterns emerge. Score-function methods appear capped: \methodGRPO{} on the flow model and diffu-GRPO on the 8B masked model improve by only a few points on any task, regardless of model size, whereas the critic-based variants achieve substantially larger, task-dependent gains. Critic resolution also matters: the 1D critic does not consistently outperform \methodGRPO{}, whereas the 2D critic outperforms both methods across all four tasks. The gains are substantial: \methodSVGTwoD{} improves Game24 by \hl{14.2} points and reaches \hl{68.7} on Sudoku, making the 105M flow model the strongest model on both tasks. Figure~\ref{fig:accuracy_flops} plots accuracy versus FLOPs across CFRL variants on all four datasets, showing that SVG-2D consistently outperforms GRR at matched compute. Scaling from ELF-B (105M) to ELF-L (652M) raises \method{}-SVG accuracy on all three non-saturated tasks (Table~\ref{tab:elf_l}). Despite stronger SFT baselines, ELF-L retains comparable or larger RL gains and outperforms \methodGRPO{}, suggesting that increased model capacity complements \method{} and supports its scaling potential.

\paragraph{Effect of the gradient estimator.}\label{sec:exp_estimator}
The right panels of Table~\ref{tab:main} test the two estimators at the same point, before any training moves the policy. We take the Game24 task-SFT model, sample $N$ complete rollouts under the deployment sampler and score them with the verifier. On these rollouts we form, at one fixed flow step (24 of 64), one gradient per rollout with respect to both the score-function gradient of \methodGRPO{} and the pathwise gradient of 2D critic. A control sets every advantage to 1, which gives an estimator whose expectation is exactly 0. GPT-2 is measured the same way on its own task-SFT model, with its token-level gradient. \emph{Top:} the $N$ rollouts are split at random into two halves, the gradients of each half are averaged, and the cosine between the two averages is plotted. Two halves of an estimator with a non-zero expectation agree more as $N$ grows, and the cosine rises towards 1; two halves of noise are unrelated, and the cosine stays near 0. \emph{Bottom:} from the base model we take one AdamW-normalized parameter update along each estimator's mean gradient over 1{,}024 rollouts; the unit $\Lambda$ is the distance in parameter space that the 2D-critic policy of Table~\ref{tab:main} travels from the base over its training. Each updated model is evaluated on the held-out prompts with the same random numbers as the base, and the change in accuracy is shown with a paired 95\% interval; a random direction of the same lengths is the control.

The score-function gradient of the flow model does not separate from the zero-mean control at any $N$ up to 2{,}048, whereas the pathwise gradient agrees with itself from 256 rollouts on. GPT-2's token gradient separates from the control at 1{,}024 rollouts: indicating that the limitation lies not only in the score-function estimator itself but also in the flow model's action space. The one-step test says the same in accuracy: a single step along the pathwise direction, at a quarter or a half of $\Lambda$, raises held-out accuracy by about 4 points, whereas the score-function step and the random step stay within about 1 point of the base at every length. The two panels tell a consistent story. Relative to its mean, the flow model's score-function gradient has high variance that independent estimates disagree and a single update has no measurable effect, even at practical rollout budgets. By contrast, pathwise estimates agree within a few hundred rollouts, and one update recovers a substantial fraction of the full-training gain. Because this difference appears before training, it reflects the estimators rather than their training recipes.

\paragraph{Effect of the self-conditioning gradient.}\label{sec:exp_selfcond}
Table~\ref{tab:selfcond} separates the two sources of SVG-2D's advantage over GRR. \emph{Left:} omitting the second gradient term of~(\ref{eq:pwgrad}), by detaching $\hat\rvx$ at the critic's input, costs SVG-2D 6.5 and 5.7 points, leaving it level with GRR on Countdown and $+6.0$ ahead on Game24. \emph{Right:} we retrain the task-SFT base of ELF-B with self-conditioning disabled, with the same supervised budget as Table~\ref{tab:main}, and rerun the whole pipeline without self-conditioning: rollout collection, critic pre-training, RL and evaluation. The base drops by about 30 points, so self-conditioning matters greatly for the model, yet SVG-2D still leads GRR by $+13.5$ and $+7.8$. The SVG gain thus comes from both reduced variance and reduced bias.

\begin{table}[!t]
\centering\small
\setlength{\tabcolsep}{3pt}
\begin{minipage}[c]{0.57\linewidth}
\centering
\begin{tabular*}{\linewidth}{@{\extracolsep{\fill}}lcccc@{}}
\toprule
 & & & \multicolumn{2}{c}{\textbf{SVG-2D}} \\
\cmidrule(l){4-5}
\textbf{Task} & \textbf{Base} & \textbf{GRR} & \textbf{detached $\hat\rvx$} & \textbf{full} \\
\midrule
Countdown & 41.8 & 42.6 & 42.7 & \textbf{49.2} \\
Game24 & 44.5 & 46.7 & 52.7 & \textbf{58.4} \\
\bottomrule
\end{tabular*}
\end{minipage}\hfill
\begin{minipage}[c]{0.39\linewidth}
\centering
\begin{tabular*}{\linewidth}{@{\extracolsep{\fill}}lccc@{}}
\toprule
\textbf{Task} & \textbf{Base} & \textbf{GRR} & \textbf{SVG-2D} \\
\midrule
Countdown & 12.7 & 13.4 & \textbf{26.9} \\
Game24 & 13.6 & 15.3 & \textbf{23.1} \\
\bottomrule
\end{tabular*}
\end{minipage}
\arrayrulecolor{black}
\caption{\textbf{Self-conditioning gradient.} Held-out accuracy (\%). \textbf{Left:} omit the second gradient term of \methodSVGTwoD{}~(\ref{eq:pwgrad}) with the estimate $\hat\rvx$ detached. \textbf{Right:} self-conditioning removed for the base model and for both SVG and GRR.}
\label{tab:selfcond}
\end{table}

\begin{table}[!t]
\centering\small
\begin{minipage}[c]{0.60\linewidth}
\centering
\setlength{\tabcolsep}{3pt}
\resizebox{\linewidth}{!}{%
\begin{tabular}{@{}lcccc@{}}
\toprule
\textbf{Critic} & \textbf{Pearson} & \textbf{pick@$V_\phi$} & \textbf{max.} & \textbf{final} \\
\midrule
\multicolumn{5}{c}{\emph{critic input and initialisation}}\\
\rowcolor{gray!15}\textbf{fine-tuned copy, reads $(y_k,x_{t_{k-1}})$} & \CDa{1} & \CDa{3} & \CDa{5} & \CDa{6} \\
\quad without $x_{t_{k-1}}$ in the input & \CDb{1} & \CDb{3} & \CDb{5} & \CDb{6} \\
\quad copy not fine-tuned & \CDc{1} & \CDc{3} & \CDc{5} & \CDc{6} \\
\quad random initialisation & \CDd{1} & \CDd{3} & 37.9 & \phantom{0}4.1 \\
\midrule
\multicolumn{5}{c}{\emph{same critic, trained differently in RL}}\\
\quad copy also trained in RL & \CDa{1} & \CDa{3} & 49.7 & 49.3 \\
\quad head not refit in RL & \CDa{1} & \CDa{3} & 49.3 & \phantom{0}0.5 \\
\bottomrule
\end{tabular}}
\end{minipage}\hfill
\begin{minipage}[c]{0.385\linewidth}
\centering
\includegraphics[width=\linewidth]{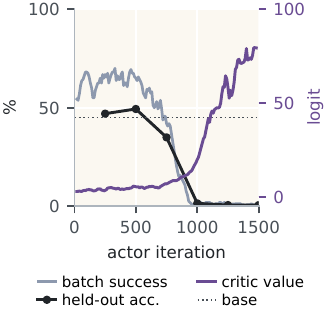}
\end{minipage}
\caption{\textbf{Critic design.} \textbf{Left:} the 2D critic on Game24 (shaded) and variants. Offline: correlation of the value with the outcome, and success of its top-ranked rollout (chance $0.55$). Online: RL accuracy (\%), max.\ and final. \textbf{Right:} the last row over training.}
\label{tab:critic_design}
\end{table}

\paragraph{Effect of the critic design.}\label{sec:exp_critic}
Table~\ref{tab:critic_design} varies one critic component at a time. Removing the estimate $x_{t_{k-1}}$or omitting backbone fine-tuning reduces offline outcome prediction and causes the policy to improve briefly before collapsing; random initialization never surpasses the base. Thus, the outcome is encoded in $x_{t_{k-1}}$ and requires pretrained denoiser features to interpret. Updating the backbone during RL provides no benefit, so joint training refits only the head. Without refit, the actor exploits the critic: its predicted value approaches certainty while the same batch's true success falls to zero.

\begin{figure}[t]
\centering
\includegraphics[width=0.85\linewidth]{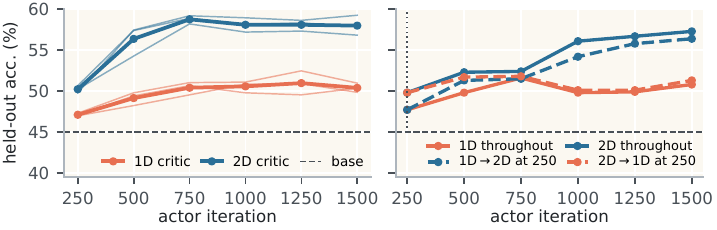}
\caption{\textbf{1D versus 2D critic.} Game24 accuracy (\%) over training. \textbf{Left:} 1D- and 2D-critic runs (thin) and their means (thick). \textbf{Right:} critics exchanged at iteration 250 (dashed) or not (solid); colour: the critic in charge afterwards.}
\label{fig:dyn}
\end{figure}

\paragraph{1D versus 2D critic.}\label{sec:exp_2d}
Figure~\ref{fig:dyn} follows the 1D- and 2D-critic runs of Table~\ref{tab:main} on Game24; the two recipes differ only in the critic's head, one reading of the pooled features against one reading per position. \emph{Left:} the 2D critic rises faster from the first checkpoint on and plateaus well above the 1D critic, which levels off early. \emph{Right:} to tell whether this is the critic's doing or a better start, the two critics are exchanged at iteration 250, while the two policies are still close. Each exchanged run ends where the critic now in charge takes it, whichever policy it inherited: the run that switches to the 2D critic climbs to the 2D level, the run that switches to the 1D critic falls to the 1D level. The advantage is therefore made during joint training, by the gradient the critic hands the actor, not by a head start; one reading per position steers the policy further than one reading of the whole state.

\paragraph{Effect of the objective, the pooling and the sampler.}\label{sec:ablations}
Table~\ref{tab:needs} takes the full 2D-critic recipe on Game24 and changes one choice at a time, keeping the base, the critic, the window and the budget fixed; each row reports how much the change costs at the best checkpoint and at the last one. Each choice matters. Ascending the success probability instead of its log-odds lowers both numbers, and the final accuracy more than the best: the probability's gradient fades on states whose outcome is nearly decided, which is exactly where the log-odds still amplify it, so the policy stops learning from its most informative states and levels off earlier. Pooling the positions by their minimum instead of their mean costs less: the gradient of a minimum reaches only the single lowest position of a state, so each update teaches fewer positions, but the positions it teaches are the right ones. Training under the cheap 16-step sampler and evaluating under the deployment sampler costs the most, two thirds of the full gain: the critic's value is defined by the sampler that produces the states, and a policy improved for one sampler is not thereby improved for another. This is why actor and critic are trained under the sampler used at deployment, and why carrying the gain to cheaper samplers needs the loop of the next paragraph rather than a change of training sampler.

\begin{table}[!t]
\centering\small
\caption{\textbf{Effect of the objective, the pooling and the sampler} on Game24: change in accuracy (\%) at the best and the final checkpoint relative to the full recipe; blue: the largest loss.}
\label{tab:needs}
\setlength{\tabcolsep}{5pt}
\begin{tabular*}{\linewidth}{@{\extracolsep{\fill}}llcc@{}}
\toprule
\textbf{Choice} & \textbf{Variant} & \textbf{$\Delta$ max.} & \textbf{$\Delta$ final} \\
\midrule
objective & actor ascends $V_\phi$ instead of $\mathrm{logit}(V_\phi)$ & $-4.0$ & $-6.1$ \\
pooling & actor ascends the minimum over positions instead of the mean & $-3.1$ & $-2.5$ \\
sampler & trained under the cheap sampler, read at deployment & \cellcolor{blue!12}$-9.4$ & \cellcolor{blue!12}$-9.4$ \\
\bottomrule
\end{tabular*}
\end{table}

\paragraph{Effect of the amortization loop.}\label{sec:exp_loop}
Table~\ref{tab:loop} starts from the 2D-critic checkpoint of Table~\ref{tab:main} on Countdown and K\&K and applies the amortization loop (Algorithm~\ref{alg:loop}) with a pool of 50k unlabelled problems: each round samples the policy on the pool, keeps the responses the verifier accepts, and fine-tunes the model on them with ELF's supervised loss, for 4 rounds on Countdown and 5 on K\&K. The left table reads each model under the cheap 16-step sampler and the 64-step deployment sampler. RL alone improves the deployment sampler, by about 8 points on Countdown and 12 on K\&K, but on Countdown it leaves the cheap sampler where the base was. The loop moves the cheap sampler most: on both tasks it ends above the accuracy the base had under the deployment sampler, a gain over the loop start of about 20 points on Countdown and 25 on K\&K, while the deployment sampler holds on Countdown and keeps rising on K\&K to \hl{99.1}, above every model of Table~\ref{tab:main}. The reason is where the two stages write their gains: RL changes the denoiser only along the transitions of the sampler it trains under, whereas ELF's supervised loss trains the denoiser at every noise level, so fine-tuning on the policy's own verified samples spreads what RL taught the policy over the whole denoiser, and every sampler reads it.

\paragraph{Effect of the loop across sampler cost.}
The right panel of Table~\ref{tab:loop} reads the same three models under four samplers ordered by cost: 16 and 32 stochastic steps, 64 deterministic ODE steps, and the 64-step deployment sampler. The base and the RL policy trace steep curves: accuracy climbs with the number of steps, and on Countdown RL's gain sits almost entirely at the expensive end, the sampler it was trained under, while on K\&K RL lifts every sampler by a similar amount. After the loop both curves are nearly flat: accuracy varies by 5 to 6 points across the four samplers. The deployment sampler is still the best, but the gap between the cheapest and the most expensive shrinks from about 20 points to about 6 on Countdown and from about 10 to about 5 on K\&K. In practice this means the 16-step sampler delivers nearly the deployment accuracy at a quarter of the denoiser calls: the loop turns a gain that belonged to one sampler into a property of the model.

\begin{table}[!t]
\centering\small
\begin{minipage}[c]{0.56\linewidth}
\centering\footnotesize
\setlength{\tabcolsep}{2pt}
\begin{tabular}{@{}l*{4}{>{\centering\arraybackslash}p{0.9cm}}@{}}
\toprule
\textbf{Model} & \multicolumn{2}{c}{\textbf{Countdown}} & \multicolumn{2}{c}{\textbf{K\&K}} \\
\cmidrule{2-3}\cmidrule{4-5}
 & \textbf{cheap} & \textbf{deploy} & \textbf{cheap} & \textbf{deploy} \\
\midrule
base & 23.8 & 42.8 & 58.6 & 68.2 \\
2D-critic ckpt (loop start) & 24.3 & 50.5 & 68.8 & 79.8 \\
\rowcolor{gray!15}\textbf{+ amortization loop} & \textbf{45.6} & 51.4 & \textbf{93.8} & \textbf{99.1} \\
\bottomrule
\end{tabular}
\end{minipage}\hfill
\begin{minipage}[c]{0.42\linewidth}
\centering
\includegraphics[width=\linewidth]{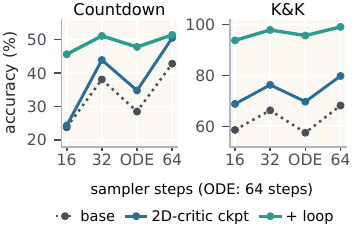}
\end{minipage}
\caption{\textbf{Amortization loop.} \textbf{Left:} accuracy (\%) under the cheap and the deployment sampler before and after the loop; bold: over 1.0 point above the loop start; shaded: the full pipeline. \textbf{Right:} the same three models across sampler cost.}
\label{tab:loop}
\end{table}

\section{Conclusion and Limitations}
We propose \method{}, RL post-training for continuous flow language models in three variants and an amortisation loop. \methodGRPO{} is the score-function estimator on the chain's exact density; \methodSVGScalar{} and \methodSVGTwoD{} are stochastic value gradients with one value per state, or per state and token position. \methodGRPO{} gains little: one reward bit is spread over the chain's noise coordinates, and its gradient agrees with itself no better than a zero-mean control, whereas GPT-2's gradient, with one score term per token decision, does. Differentiating the chain through a critic removes the noise factor. The critic must read the denoiser's estimate through a copy of the pre-trained denoiser and be refit after every actor step, or it is exploited. With one value per position the policy rises faster and further than with a scalar value, and exchanging the critics mid-training shows that the critic in charge decides. The loop carries the gain into a cheap chain.

\paragraph{Limitations.}
Our results cover four synthetic reasoning tasks with exact verifiers; natural-language tasks and learned rewards remain untested. The RL gain is tied to the sampler used in training, and carrying it to cheaper samplers requires the amortization loop and a pool of unlabelled problems. Finally, the autoregressive reference still leads on two of the four tasks, so the advantage of a flow model with RL is task-dependent rather than uniform.

\subsubsection*{Reproducibility statement}
Code, configurations and job scripts for every run in this paper will be released; all datasets are generated programmatically or drawn from public collections (Appendix~\ref{app:datasets}) and the exact splits, seeds and verifier implementations are included. Appendix~\ref{app:datasets} defines the samplers, Appendix~\ref{app:details} lists every hyperparameter, the evaluation protocol and the compute used; Appendix~\ref{app:results} reports the full trajectories and the supporting results behind the numbers in the main text.

\subsubsection*{AI use statement}
In this work, we used generative AI tools for writing assistance (editing prose), for writing and debugging experiment code. We have not used generative AI tools for generating research ideas that are presented as contributions, nor for producing any experimental result, number or figure that was not computed by the released code. We have reviewed all AI-assisted work: all code was tested against the reported runs, all citations were verified against the cited sources, and all claims were checked against the experiment logs. The authors take full responsibility for all content of this paper.

\bibliography{refs}
\bibliographystyle{iclr2027_conference}

\appendix
\raggedbottom 
\section{Datasets and verifiers}
\label{app:datasets}
All tasks are generated programmatically; every prompt has a program verifier that decides correctness of a model output without reference to the gold answer, so rewards are exact and label-free.

\paragraph{Countdown.}
Following the Countdown formulation of \citet{gandhi2024sos}, a prompt is four or five integers and a target; the model emits a derivation as a sequence of binary operations (\texttt{12-5=7 , 4*7=28 , 28-3=25} for the numbers 12, 5, 4, 3 and the target 25). The verifier parses the derivation, checks that every operation is arithmetically correct, that each input number is used exactly once, and that the last result equals the target. The supervised set is 500k examples, one gold derivation each; the loop's pool is a disjoint set of 50k unlabelled problems; evaluation uses 2{,}000 held-out problems. Sequence length is 72 embedding positions (16 for the condition).

\paragraph{Sudoku.}
Puzzles are drawn from a public collection of rated Sudoku puzzles (with solutions); the model reads the 81-cell puzzle (blanks as \texttt{*}) and emits all 81 cells. The verifier checks that the output is a valid grid consistent with the givens; since solutions are unique this equals exact match. The supervised set is 400k puzzles, evaluation 2{,}000 held-out puzzles; the ``fresh puzzle'' pool for the loop is 50k further puzzles with no overlap with training, validation or test. Sequence length is 164 (82 condition, 82 target).

\paragraph{Knights-and-Knaves.}
Following \citet{xie2024kk}, each puzzle has $n\in\{4,5,6\}$ characters who each make one statement (``A: C knave'', ``B: A same C'', ``C: D diff B''); the model outputs an assignment (``A knight B knave \dots''). Our generator enumerates assignments and keeps only puzzles with a unique solution; the verifier re-parses the statements and checks that the assignment satisfies every one. Training uses 300k puzzles; evaluation 2{,}000; the loop's first round samples the 50k training puzzles without their labels, and every later round a fresh 50k pool disjoint from training, validation and each other. The supervised model reaches 96.1/93.3/99.4/99.7 (ODE-64/$\gamma4$@16/$\gamma28$@32/$\gamma56$@64) at epoch 8, so the RL base is the epoch-4 checkpoint (57.5/58.6/66.3/68.2) and the epoch-8 model is the supervised oracle. Sequence length is 96 (56 condition).

\paragraph{Game24.}
Following the Game of 24 of \citet{yao2023tot}, reach 24 from four numbers in $\{1,\dots,13\}$ with each number used once. Only 1{,}345 unordered number sets are solvable, so we split by set: the training set contains 13{,}150 orderings of 1{,}076 sets with one gold derivation each, the loop pool a further 4{,}693 orderings of the training sets, and the evaluation set 1{,}601 orderings of the 269 remaining sets, which the model never sees in any order. The verifier is the Countdown verifier with target 24. Sequence length is 72 (16 condition). The supervised model is initialised from the Countdown model and trained for 40 epochs; held-out accuracy peaks at epoch 20 (36.2/35.6/43.5/44.0 at ODE-64/$\gamma4$@16/$\gamma28$@32/$\gamma56$@64) and falls to 30.4/29.0/32.4/34.1 by epoch 40, so the epoch-20 checkpoint is the RL base and there is no supervised oracle above it.

\paragraph{Samplers.}
ODE-64 is the 64-step deterministic sampler of \citet{hu2026elf}. $\gamma56$@64 denotes the stochastic sampler of~\eqref{eq:sde} with $K{=}64$ uniform steps and noise re-injection scale (churn) $\gamma{=}56$, and likewise for the other values; $\gamma56$@64, $\gamma28$@32 and $\gamma14$@16 all have per-step retention $\alpha\approx0.125$, which we found to be the best sampling configuration independently of $K$ (more steps still help at fixed $\alpha$); $\gamma4$@16 is a cheap low-churn configuration. All evaluations use classifier-free guidance 2 and self-conditioning guidance 1 as in the ELF release. Throughout, the \emph{deployment sampler} is $\gamma56$@64 ($\gamma28$@32 on Sudoku) and the \emph{cheap sampler} is $\gamma4$@16 ($\gamma14$@16 on Sudoku).

\section{Derivations}
\label{app:derivations}

\subsection{The stochastic sampler and its transition density}
\label{app:kernels}
The sampler of~(\ref{eq:sde}) is the stochastic sampler released with ELF \citep{hu2026elf}. At every step it re-noises the current state, queries the denoiser at the re-noised state and at the correspondingly lowered time, and moves part of the way from the re-noised state towards the estimate; the coefficients $\lambda_k$, $\eta_k$ and $\psi_k$ of~(\ref{eq:sde}) follow from these three operations, as derived below. The state of the MDP, $s_k=(\rvz_{t_k},\rvx_{t_{k-1}})$, is the pair the denoiser reads at its $(k{+}1)$-th call: the re-noised embeddings and the previous estimate carried by self-conditioning.

\paragraph{Derivation from the released implementation.}
The official implementation of ELF uses the following stochastic sampler with churn $\gamma$
\citep{hu2026elf}. The implementation is slightly different from Algorithm 6 in the original
paper. Here we provide the derivation based on the code, and the notations are
mostly consistent with the original paper.

At step $k$, let the current time be $t_{k-1}$ and the step size be $h$, so that
$t_{k} = t_{k-1} + h$. Let $\gamma \ge 0$ be the noise re-injection scale, and
$t_{\text{eps}}=0.05$ is the lower bound used in the velocity conversion.
Define
\begin{equation}
  \alpha = \operatorname{clip}\!\left(1 - \gamma h,\; 0,\; 1\right),
  \quad
  \tilde{t}_{k-1} = \alpha t_{k-1},
  \quad
  D = \max\!\left(1 - \tilde{t}_{k-1},\; t_{\text{eps}}\right)
    = \max\!\left(1 - \alpha t_{k-1},\; t_{\text{eps}}\right).
\end{equation}

With $\epsilon \sim \mathcal{N}(0, c_\sigma^2I)$, one step performs
\begin{align}
  \tilde z_{k-1} &= \alpha\, z_{k-1} + (1 - \alpha)\, \epsilon,
    \label{eq:renoise} \\
  x_{k-1} &= \mathrm{d}_\theta\!\left(\tilde z_{k-1},\, \tilde t_{k-1}, x_{k-2}, c\right),
    \label{eq:predict} \\
  v_{k-1} &= \frac{x_{k-1} - \tilde z_{k-1}}{D},
    \label{eq:velocity} \\
  z_k &= \tilde z_{k-1} + \left(t_k - \tilde{t}_{k-1} \right) v_{k-1}.
    \label{eq:update} \\
  \tilde{z}_k &= \alpha\, z_{k} + (1 - \alpha)\, \epsilon.
    \label{eq:renoise2}
\end{align}
Note that $x_{k-1}$ is the prediction at the re-noised input
$(\tilde z_{k-1}, \tilde t)$, not at $(z_{k-1}, t)$. At each step,
the code record $(z_k, x_{k-1})$ at each step. At $k = 1$, we
use $z_{0} = \varepsilon$ and $x_{-1} = 0$.

In our derivation, we use $(\tilde{z}_{k}, x_{k-1})$ as the state, which is $(\rvz_{t_k}, \rvx_{t_{k-1}})$ in our notation. We write:
\begin{align}
\tilde{z}_k = \lambda_k \tilde{z}_{k-1} + \eta_k x_{k-1} + \psi_k \epsilon
\end{align}
with
\begin{align}
\lambda_k &= \alpha \left(1 - \frac{t + h - \tilde t}{D}\right), \quad
\eta_k = \frac{\alpha (t + h - \tilde t)}{D}, \quad
\psi_k = 1 - \alpha,
\end{align}
for $k =1, \ldots, K - 1$. When $k = K$, we set $\alpha=1$ and eliminate the re-noising step,
\begin{align}
\lambda_K &= \left(1 - \frac{t + h - \tilde t}{D}\right), \quad
\eta_K = \frac{(t + h - \tilde t)}{D}, \quad
\psi_K = 0,
\end{align}
so that $\tilde{z}_K = z_K$, matching the ELF implementation.

To derive these, substitute \eqref{eq:velocity} into \eqref{eq:update} and group the terms by $\tilde z_{k-1}$ and $x_{k-1}$:
\begin{align*}
z_k &= \tilde z_{k-1} + \frac{t + h - \tilde t}{D}\left(x_{k-1} - \tilde z_{k-1}\right) \\
    &= \left(1 - \frac{t + h - \tilde t}{D}\right)\tilde z_{k-1} + \frac{t + h - \tilde t}{D}\, x_{k-1}.
\end{align*}
Substituting this into \eqref{eq:renoise2} gives
\begin{align*}
\tilde z_k = \alpha\left(1 - \frac{t + h - \tilde t}{D}\right)\tilde z_{k-1} + \frac{\alpha\,(t + h - \tilde t)}{D}\, x_{k-1} + (1 - \alpha)\,\epsilon,
\end{align*}
from which $\lambda_k$, $\eta_k$ and $\psi_k$ are read off. The noise enters only through the re-noising after the update, so $\psi_k$ carries no factor from the step.

Two consequences are used in the main text. First, the time fed to the denoiser together with $\rvz_{t_k}$ is the lowered time $\tilde t_k=\alpha\,t_k$ of the re-noised state, which the main text writes as $t_k$ in~(\ref{eq:denoiser}) for brevity. Second, given $s_{k-1}$, the next embeddings are Gaussian: for $k=1,\dots,K-1$ and every generated position $i\in\mathcal G$,
\begin{equation}
p_\theta\big(\rvz_{t_k}[i]\,\big|\,s_{k-1}\big)=\mathcal N\!\big(\rvz_{t_k}[i];\ \lambda_k\rvz_{t_{k-1}}[i]+\eta_k\,\rvx_{t_{k-1}}[i],\ \psi_k^2c_\sigma^2I_d\big),\qquad \rvx_{t_{k-1}}=\mathrm d_\theta(s_{k-1}),
\label{eq:kernel}
\end{equation}
independently across the generated positions $i\in\mathcal G$ (the response positions), while the prompt positions are held fixed; the last step has $\psi_K=0$ and is deterministic, and the decoder gives $\pi_\theta(a_K\mid s_K)=\prod_i\mathrm{softmax}(W\rvz_{t_K}[i])_{w_i}$. The mean of~(\ref{eq:kernel}) depends on $\theta$ only through the estimate $\mathrm d_\theta(s_{k-1})$, so the $\theta$-dependent part of $\log\pi_\theta(a_k\mid s_k)$ in~(\ref{eq:grpo_loss}) is the log-density~(\ref{eq:kernel}) of the next state, and the trajectory log-likelihood is the sum of these terms and the decoder term, with no bound involved. We verified this on CPU: recomputing the log-likelihood of a stored trajectory under the parameters that produced it agrees with the value recorded at rollout to relative error $7\times10^{-5}$ (denoising steps) and $10^{-3}$ (decoder). One caveat matters in practice: $\gamma\Delta_k=1$ makes $\alpha=0$, in which case the state is replaced by fresh noise and every parameter gradient of that step vanishes, so we require $\alpha>0$.

\subsection{Score-function and pathwise gradients of one transition}
\label{app:variance}

Action $a_k$ is taken from $s_k$ and produces $s_{k+1}$. For a stochastic step,
\[
a_k=\eta_k\rvx_{t_k}+\psi_k\epsilon_k,
\qquad \rvz_{t_{k+1}}=\lambda_k\rvz_{t_k}+a_k,
\qquad \rvx_{t_k}=\mathrm d_\theta(s_k).
\]
Thus $\pi_\theta(a_k\mid s_k)=\mathcal N(a_k;\eta_k\mathrm d_\theta(s_k),\psi_k^2c_\sigma^2I)$. For fixed $s_k$, the action and next-embedding densities differ only by the parameter-independent translation $\lambda_k\rvz_{t_k}$ and have the same score. The local action score, with the input state held fixed, is
\begin{equation}
\partial_\theta\log\pi_\theta(a_k\mid s_k)=\frac{\eta_k}{\psi_k c_\sigma^2}\sum_{i}\epsilon_{k}[i]^{\top}\,\frac{\partial\,\mathrm d_\theta(s_k)[i]}{\partial\theta},
\label{eq:sf}
\end{equation}

where $\epsilon_k\sim\mathcal N(0,c_\sigma^2I)$ is the noise injected at that step: the advantage multiplies the injected noise before it passes through the denoiser's Jacobian. Summed over the stochastic steps and weighted by the group-relative advantage, Equation~\ref{eq:sf} gives the local denoising-action gradient used to optimize Equation~\ref{eq:grpo_loss}; the decoder term is treated separately. A trajectory therefore contributes $(K-1)\cdot|\mathcal G|\cdot d$ terms, each proportional to an independent Gaussian coordinate and weighted by an advantage that depends on the trajectory only through its final reward. A coordinate adds to the expected gradient only through the covariance of the reward with that coordinate's noise, but it adds to the variance regardless; with $N$ trajectories per update the signal grows like $N$ times these covariances and the noise like $\sqrt{N(K-1)|\mathcal G|d}$. In an autoregressive model the analogous sum has one term per token and every term is a decision the reward depends on; in the flow model the covariances are, as far as we can measure, undetectable (the half-split cosine of Table~\ref{tab:main} stays at the level of a zero-mean control up to $2{,}048$ rollouts), which is the origin of the gap between the two estimators in Table~\ref{tab:main}. A finer advantage, per step or per position, changes which terms receive weight but leaves the factor $\epsilon_k/(\psi_kc_\sigma^2)$ in place.

We now expand the denoising-action scores. By the chain rule,
\begin{equation}
D_\theta \rvx_{t_{k}}
=\underbrace{\partial_\theta\mathrm d_\theta(s_k)}_{\text{Local gradient}}
+\underbrace{\big(\partial_{\rvx_{t_{k-1}}}\mathrm d_\theta(s_k)\big)D_\theta \rvx_{t_{k-1}}}_{\text{Earlier self-conditioning path}},
\label{eq:self_conditioning_chain}
\end{equation}
where $D_\theta$ is the total derivative with respect to $\theta$ through the deterministic self-conditioning path, with $D_\theta \rvx_{t_{-1}}=0$. The local derivative $\partial_\theta\mathrm d_\theta(s_k)$ holds the entire input state fixed, including the stored self-conditioning input. Since the Gaussian action has mean $m_k=\eta_k \rvx_{t_{k}}$ and covariance $\sigma_k^2 I$, where $\sigma_k=\psi_k c_\sigma>0$, its total score is
\begin{equation}
\begin{aligned}
D_\theta\log\pi_\theta(a_k\mid s_k)
={}&\underbrace{\frac{\eta_k}{\sigma_k^2}(a_k-m_k)^\top\partial_\theta\mathrm d_\theta(s_k)}_{\text{GRR local score}}\\
&+\underbrace{\frac{\eta_k}{\sigma_k^2}(a_k-m_k)^\top
\big(\partial_{\rvx_{t_{k-1}}}\mathrm d_\theta(s_k)\big)D_\theta \rvx_{t_{k-1}}}_{\text{Omitted trajectory score}}.
\end{aligned}
\label{eq:trajectory_score_decomposition}
\end{equation}
The vector products sum over all generated positions and embedding coordinates. Because $a_k-m_k=\psi_k\epsilon_k$, the first term is precisely the local score in Equation~\ref{eq:sf}.

Without parameter-dependent deterministic self-conditioning, the omitted term vanishes, so local differentiation recovers the full denoising-action contribution.

With self-conditioning, recovering the exact trajectory-level score additionally requires differentiating each transition log-density through preceding denoiser predictions, increasing computational and memory costs. We therefore retain the local transition gradient and omit these cross-step dependencies as a computationally motivated approximation. Related approximations appear in diffusion-model optimization: DPOK omits the contribution of earlier transitions to the state-distribution derivative of its expected KL regularizer for efficiency \citep[Section~4.1 and Appendix~A.3]{fan2023dpok}, while DRaFT-$K$ truncates backpropagation to the final $K$ sampling steps \citep{clark2023draft}, where $K$ denotes its truncation length.

The pathwise estimator of $\mathcal J_{\rm SVG}$ uses the same reparameterisation, $\rvz_{t_{k+1}}=\lambda_k\rvz_{t_k}+\eta_k\,\mathrm d_\theta(s_k)+\psi_k\epsilon_k$ and $\rvx_{t_k}=\mathrm d_\theta(s_k)$, and differentiates the critic through it with $s_k$ and $\epsilon_k$ held fixed:
\begin{equation}
\nabla_\theta V_\phi(s_{k+1})=\Big[\eta_k\,\nabla_{\rvz}V_\phi(s_{k+1})+\nabla_{\rvx}V_\phi(s_{k+1})\Big]\frac{\partial\,\mathrm d_\theta(s_k)}{\partial\theta},
\label{eq:pwgrad}
\end{equation}
the first channel through the re-noised embeddings, the second through the estimate that self-conditioning feeds to the next call. No noise factor appears: the weight on each coordinate is the critic's derivative there, which vanishes on every coordinate the predicted outcome does not depend on. One random $k$ per update suffices (next section), so the sampler need not be differentiated end to end and memory is constant in $K$.

\subsection{Why differentiating one transition suffices}
\label{app:pw_derivation}
Fix a prompt and write $V^\star_k(s)=\mathbb E[R(s_{K+1})\mid s_k=s]$ as in Section~\ref{sec:method}. Since the objective~(\ref{eq:objective}) is $J(\theta)=\mathbb E[V^\star_k(s_k)]$ for every $k$ by the tower property, and $s_k$ is produced from $s_{k-1}$ by the call $\mathrm d_\theta(s_{k-1})$ and the injected noise, the value satisfies the one-step recursion
\begin{equation}
V^\star_{k-1}(s)=\mathbb E_{\epsilon_k}\big[V^\star_k(s_k)\big],\qquad s_k=\big(\lambda_k\rvz_{t_{k-1}}+\eta_k\,\mathrm d_\theta(s)+\psi_k\epsilon_k,\ \mathrm d_\theta(s)\big),\qquad s=(\rvz_{t_{k-1}},\rvx_{t_{k-2}}),
\label{eq:bellman}
\end{equation}
and $V^\star_k$ depends on $\theta$ only through the calls at $s_k$ and later. The parameters enter $J$ once per denoiser call. Writing $\nabla^{(k)}_\theta$ for the gradient taken through the call at $s_{k-1}$ alone, with every other call held fixed, the chain rule gives
\begin{equation}
\nabla_\theta J=\sum_{k=1}^{K+1}\nabla^{(k)}_\theta J,\qquad \nabla^{(k)}_\theta J=\mathbb E\big[\nabla^{(k)}_\theta V^\star_k(s_k)\big]\quad(1\le k\le K-1),
\label{eq:sumover}
\end{equation}
where the second equality exchanges derivative and expectation, and $\nabla^{(k)}_\theta V^\star_k(s_k)$ is~(\ref{eq:pwgrad}) with $V^\star_k$ in place of $V_\phi$. The exchange needs $V^\star_k$ to be locally Lipschitz in $s_k$ with an integrable gradient. For $k\le K-1$ this holds in the first argument because, by~(\ref{eq:bellman}), $V^\star_k$ is a Gaussian average of $V^\star_{k+1}$ over the noise injected at step $k{+}1$; for the second argument, the estimate carried by self-conditioning, we assume the same regularity. It fails at the last call: $V^\star_K$ is the reward of the decoded response, an $\arg\max$ followed by a program, hence piecewise constant with zero gradient almost everywhere, so a pathwise estimator sees no signal through the final decode. This is why the pathwise route needs a value on \emph{intermediate} states rather than the reward on the final one, and why $\mathcal J_{\rm SVG}$ is written with a critic at $s_k$. Each term of~(\ref{eq:sumover}) is the one-step value gradient of SVG \citep{heess2015svg} with the sampler as the transition model; drawing one $k$ at random and applying~(\ref{eq:pwgrad}) is an unbiased estimate of the window's share of the sum, up to the constant $|\mathcal W|$, when $V_\phi=V^\star_k$.

The first channel of~(\ref{eq:pwgrad}) and the score-function term~(\ref{eq:sf}) estimate the same quantity. By Gaussian integration by parts, $\mathbb E[\psi(\mu+\sigma\epsilon)\,\epsilon]=\sigma\,\mathbb E[\nabla\psi(\mu+\sigma\epsilon)]$ for standard Gaussian $\epsilon$ and locally Lipschitz $\psi$ with integrable gradient; applied to $\psi=V^\star_k(\cdot,\rvx_{t_{k-1}})$ it turns the factor $V^\star_k(s_k)\,\epsilon_k/(\psi_kc_\sigma)$ of the score-function form into the factor $\nabla_{\rvz}V^\star_k(s_k)$ of the pathwise form, conditionally on $s_{k-1}$. The two have the same expectation and differ only in variance, coordinate by coordinate: the score-function factor at a coordinate has variance of order $\mathbb E[(V^\star_k)^2]/(\psi_kc_\sigma)^2$ whether or not the coordinate affects the reward, whereas the pathwise factor is identically zero on every coordinate the value does not depend on. The second channel of~(\ref{eq:pwgrad}) has no counterpart in~(\ref{eq:grpo_loss}), which holds the stored estimates fixed. With a learned critic in place of $V^\star_k$ the estimate is biased by the difference of the two slopes on the states the current policy visits; both the visited states and $V^\star_k$ itself move with $\theta$, which is why the critic is refit on the current policy's own rollouts after every policy update (Section~\ref{sec:jpc_2d}).

\subsection{The logit surrogate for the actor objective}
\label{app:logit_surrogate}
In the actor's training objective, Eq. (7) maximizes the expected reward $\mathbb{E}[V_\phi]$. Given that $V_\phi(s_k)\approx P(R=1\mid s_k)$, the actor is trained to maximize the expected reward. In our SVG implementation, we use $\mathbb{E}[\mathrm{logit}(V_{\phi})]$ instead. This design encourages the actor to stay away from wrong answers, even when its output is correct. It shares the same spirit as MaxRL \citep{tajwar2026maximum}.

Maximizing the logit of $V_\phi$ encourages the generator to move toward states leading to successful trajectories. Furthermore, the two objectives give the same per-sample gradient direction, since $\nabla_\theta\mathrm{logit}(V_\phi)=\nabla_\theta V_\phi/\big(V_\phi(1-V_\phi)\big)$. However, directly maximizing $V_\phi$ introduces a multiplicative factor $V_\phi(1-V_\phi)$ in the pathwise gradient, which attenuates the learning signal when the critic confidently predicts success or failure. The logit transformation removes this sigmoid-induced attenuation while preserving the per-sample gradient direction, providing a stronger optimization signal, particularly for samples predicted to fail with high confidence. Table 3 confirms the effect: ascending $V_\phi$ instead lowers accuracy. Table~\ref{tab:needs} confirms the empirical benefit: ascending $V_\phi$ instead lowers accuracy.

\subsection{Head before or after the mean}
\label{app:additive}
The 1D critic~(\ref{eq:value1}) applies the head to the mean of the position features, the 2D critic~(\ref{eq:value2}) applies it to every position and pools the readings. The following statement makes precise the bottleneck of Section~\ref{sec:jpc_2d}: with the features held fixed, a head on the mean can weigh the positions only linearly.
\begin{proposition}
\label{prop:additive}
Let $m\ge2$, $\rvh=(h_1,\dots,h_m)$ with $h_i\in\mathbb R^{d'}$, $\bar h=m^{-1}\sum_ih_i$, and let $g_1,g_2:\mathbb R^{d'}\to\mathbb R$ be differentiable.
\textup{(a)} If $g_1(\bar h)=m^{-1}\sum_ig_2(h_i)$ for all $\rvh$, then $g_1=g_2$ and both are affine.
\textup{(b)} For $U=g_1(\bar h)$, $\partial U/\partial h_j=m^{-1}\nabla g_1(\bar h)$ for every $j$; for $U=m^{-1}\sum_jg_2(h_j)$, $\partial U/\partial h_j=m^{-1}\nabla g_2(h_j)$.
\end{proposition}
\begin{proof}
(a) Differentiating the identity with respect to $h_1$ gives $\nabla g_1(\bar h)=\nabla g_2(h_1)$ for all $\rvh$. Fix $h_1$ and vary $h_2$: $\bar h$ ranges over all of $\mathbb R^{d'}$ while the right-hand side is fixed, so $\nabla g_1$ is constant and $g_1$ is affine. Setting $h_1=\dots=h_m=u$ in the identity gives $g_2(u)=g_1(u)$. (b) is the chain rule.
\end{proof}
Part (b) is the gradient statement: through a head on the mean, every position's feature is pushed along one direction, the head's slope at $\bar h$; through a head per position, position $j$ is pushed along the head's slope at its own reading. The proposition holds the features fixed; in the critic they are computed by a transformer that attends over the whole state, so neither order limits what the critic can represent as a function of the state. The two critics also share their population target: in expectation over the states at step $k$ and their outcomes, both terms of~(\ref{eq:critic2d_loss}) are minimised by the position-independent field $f_\phi(h_i)=\mathrm{logit}\,V^\star_k(s_k)$, whose pooled value is the minimiser of~(\ref{eq:critic_loss}) as well. Their difference is therefore one of parametrisation and of behaviour under joint training, which is how Table~\ref{tab:main} and Figure~\ref{fig:dyn} report it.

\subsection{Token commitment along the trajectory}
\label{app:commit}
For each generated position we record the first step at which the estimate $\rvx_{t_k}$ decodes to the final token and never changes afterwards. On Countdown, 86\% of positions commit at the first denoiser call (median the first call, 90th percentile the third), identically for $\gamma\in\{4,8,14\}$; the remaining contested positions commit uniformly along the trajectory. Most word decisions are therefore made in the first forward pass and carried by the self-conditioned estimate. This statistic describes the decoded tokens, not the coordinates' influence on the reward (Appendix~\ref{app:variance}).

\section{Method details}
\label{app:method_details}
This appendix holds the parts of the method that the main text states briefly: the choices behind the pathwise estimator and the critic (Sections~\ref{sec:method} and~\ref{sec:jpc_2d}), the design of the joint training (Algorithm~\ref{alg:jpc}), and the amortization loop in full. Notation is that of the main text: $s_k=(\rvz_{t_k},\rvx_{t_{k-1}})$, the denoiser $\mathrm d_\theta$, the sampler coefficients $\lambda_k,\eta_k,\psi_k$ of~(\ref{eq:sde}), the reward $R(s_{K+1})$ of the decoded response, and $\mathcal G$ for the set of generated positions.

\subsection{The pathwise estimator and the critic}
\label{app:pw_details}
\paragraph{One transition.}
The estimator of $\mathcal J_{\rm SVG}$ draws one transition $k$ uniformly from a window $\mathcal W=\{k_{\rm lo},\dots,k_{\rm hi}\}$ of the sampler, with $1\le k_{\rm lo}$ and $k_{\rm hi}\le K-1$, and differentiates the critic through that transition alone:
\begin{equation*}
\hat g_{\rm PW}=\nabla_\theta V_\phi(s_k),\qquad k\sim\mathcal U(\mathcal W),\qquad s_k=\big(\lambda_k\rvz_{t_{k-1}}+\eta_k\,\mathrm d_\theta(s_{k-1})+\psi_k\epsilon_k,\ \mathrm d_\theta(s_{k-1})\big),
\end{equation*}
with $s_{k-1}$ and $\epsilon_k$ held fixed (eq.~\ref{eq:pwgrad}). It is computed by rolling the sampler to $s_{k-1}$ without gradients and taking the denoiser call at $s_{k-1}$ with gradients, so its memory is that of the backward passes through the actor's call (two evaluations under classifier-free guidance) and through the critic, independent of $K$: the dependence of the reward on $s_k$ through the remaining steps is carried by the value in expectation rather than by back-propagation through them (Appendix~\ref{app:pw_derivation}). The window is the middle of the trajectory, $\mathcal W=[K/8,\,5K/8]$ on every task: it excludes the first steps, where the state is noise and the value flat, and the last, where the state is committed and the reward of the decoded response is piecewise constant; its ends are hyperparameters, not derived. With a learned critic the update is biased by the difference between the critic's slope and the true value's slope on the states the current policy visits, which is why the critic is refit on those states after every actor step (Algorithm~\ref{alg:jpc}). The update has no baseline, advantage or trust region: a baseline is a function of $s_{k-1}$, which the differentiated call holds fixed, and so drops from a pathwise gradient, and the policy's drift is controlled by the step size and by the refit of the critic.

\paragraph{Training target.}
No intermediate labels are needed. States are taken from rollouts of the sampler together with the rollout's reward $R(s_{K+1})$, and the critic is trained by binary cross-entropy, (\ref{eq:critic_loss}) for the 1D critic and~(\ref{eq:critic2d_loss}) for the 2D critic; the minimiser of either over all functions of the state is the success probability $V^\star_k(s_k)=p_\theta(R(s_{K+1})=1\mid s_k)$. \emph{Pre-training.} Before RL the critic is trained on rollouts of the task-SFT policy, with states stored at several steps spanning the trajectory and each state labelled by its rollout's reward; the copy of the denoiser and the head $f_\phi$ are trained together for a fixed number of epochs, and the epoch is chosen by a held-out score of the critic (Appendix~\ref{app:details}). The rollouts are drawn under the training sampler, or, when its early states carry no information about the outcome, under a cheaper sampler of the same denoiser, whose states the critic reads at their own time. \emph{During RL} the copy is frozen and only the head is trained. This is a cost and stability trade-off: refitting the copy after every actor step would cost a full fine-tuning step and would move the actor's objective between actor steps, whereas the head is a small correction on fixed features that a few steps refit.

\paragraph{Per-position dependence is not restricted to be affine.} With the features held fixed, if a head on the mean feature could be written as a mean of per-position readings, it would be affine (Proposition~\ref{prop:additive}(a)): a head on the mean weighs the positions only linearly, through the single bottleneck $\bar h$; a head applied per position may depend on each reading through any function, for instance one that saturates when a reading is confident.

\subsection{Design choices of the joint training}
\label{app:joint_details}
Algorithm~\ref{alg:jpc} refits the head after every actor step. Its remaining choices are the following trade-offs.

\emph{Where the head is refitted.} Each iteration rolls out a fresh batch of prompts to the end of the trajectory, decodes and verifies them, and pushes their state at the top of the window, $s_{k_{\rm hi}}$, with the reward into a replay buffer; the head is then trained on minibatches from the buffer with~(\ref{eq:critic2d_loss}) (its first term for the 1D critic). The refit is made at one step only, the top of the window, where the outcome is most determined by the state, so that the head follows the policy's changes rather than the residual randomness of the outcome. The critic at the earlier steps of the window is the pre-trained one, corrected only through what is shared, the head and its time input, on features the joint phase does not change. This is a design choice for the stability of the refit, not a consequence of the objective, and its cost is that the check below tests the critic at $k_{\rm hi}$ only.

\emph{Staleness of the buffer.} The buffer holds the newest $N_{\mathcal B}$ states, pushed $B$ per iteration, so it spans $N_{\mathcal B}/B$ iterations of past policies and the head is fitted to a mixture of their values. This is a second approximation, made for the sake of the sample size of the refit; with the values of Appendix~\ref{app:details} the span is $512$ iterations, a third to a half of a run, so the head is fitted to a genuine mixture of past policies rather than to the current one. What makes it usable is that the actor's step is small and that the check below is made on the newest states only.

\emph{The critic stays ahead of the actor.} A fixed number $m_{\rm c}$ of head steps per actor step suffices while the policy moves slowly. Every $n_q$ iterations the critic's discrimination is measured on the newest $N_q$ buffer states by the AUC, the area under the ROC curve of $V_\phi$ as a score for the reward (the probability that a random successful state scores above a random failed one). If it falls below a floor $a_0$, the critic takes $m'_{\rm c}$ extra steps before the actor moves again. The AUC at $k_{\rm hi}$ is a proxy for a necessary condition, that the critic still separates outcomes under the current policy, not a test of slope agreement across the window.

\emph{Step size.} The bias of a learned critic is paid at every actor step in proportion to the step, which motivates a decreasing step; the schedule itself, a cosine decay to a tenth of the initial value, is a choice. The critic's learning rate is constant.

\emph{One sampler.} The value function is defined by the sampler: a different sampler ($\gamma$, $K$) induces different trajectories and different values, so critic and actor are trained under the sampler that will be deployed, with one exception: on Countdown the critic's \emph{pre-training} uses states of the cheap sampler, because the deployment sampler's early and middle states carry almost no information about the outcome (Appendix~\ref{app:details}); the joint refit and the actor still run on the deployment sampler. Appendix~\ref{app:loop} also uses the critic to rank and nudge candidates on a cheaper sampler; that use rests on the assumption that a state's reading, at its own time, ranks candidates on that sampler too, which the training argument does not cover.

\subsection{The amortization loop}
\label{app:loop}
The value and the critic are defined by the sampler under which the policy was trained, and the RL improvement is made along that sampler's trajectories; a cheaper sampler, with fewer steps or $\gamma=0$, runs different trajectories through the same denoiser, and an improvement along one need not carry over to the other. The loop moves the improvement into the model's supervised objective, which trains the denoiser at every time $t$ rather than along one sampler, and uses the critic of \methodSVGTwoD{} a second time, to steer generation before the verifier can.

\paragraph{One round.}
Let $p_\theta(\rvw\mid\rvc)$ be the distribution of decoded responses under the sampler that draws the round's samples. A round takes a pool $\mathcal P$ of prompts without labels, disjoint from the evaluation prompts, draws $n$ responses per prompt, keeps the verified ones, and fine-tunes the model on them:
\begin{align}
\mathcal A(\theta)&=\Big\{(\rvc,\rvw^{(j)}):\ \rvc\in\mathcal P,\ 1\le j\le n,\ R(\rvw^{(j)})=1\Big\},\qquad \rvw^{(1)},\dots,\rvw^{(n)}\overset{\text{iid}}{\sim}p_\theta(\cdot\mid\rvc),
\label{eq:accept}\\
\theta'&=\text{the point reached by descending }\mathcal L_{\rm SFT}\big(\vartheta;\mathcal A(\theta)\big)\text{ from }\vartheta=\theta,
\label{eq:sft_step}
\end{align}
where $\mathcal A(\theta)$ is first reduced in place, by removing duplicate responses of a prompt and keeping at most $M$ per prompt (the reduced set is the one~(\ref{eq:sft_step}) and Algorithm~\ref{alg:loop} use), the descent runs for a fixed number of epochs, and $\mathcal L_{\rm SFT}$ is ELF's training loss on the accepted pairs, the objective that produced the task-SFT model \citep{hu2026elf}: for $(\rvc,\rvw)\in\mathcal A$ with clean embeddings $\rvx$ of $\rvw$, a random time $t$ and noise $\epsilon$ on the generated positions, $\rvz_t=(1-t)\epsilon+t\rvx$ on $\mathcal G$ with the prompt positions held at $\rvc$,
\begin{equation}
\mathcal L_{\rm SFT}(\vartheta;\mathcal A)=\mathbb E_{(\rvc,\rvw)\in\mathcal A}\Big[\,\mathbb E_{t,\epsilon}\sum_{i\in\mathcal G}\big\|\mathrm d_{\vartheta}(\rvz_t,t,\rvx_{\rm prev},\rvc)[i]-\rvx[i]\big\|^2-\sum_{i\in\mathcal G}\log p_{\vartheta}(w_i\mid\rvx)\Big],
\label{eq:sft}
\end{equation}
the regression of the clean embeddings on the generated positions plus the cross-entropy of the decoder on the clean embeddings. The self-conditioning input $\rvx_{\rm prev}$ is, as in ELF's training, the network's own estimate from a preceding call without gradient for half of the examples and zero for the rest, and the prompt is dropped at random for guidance training; both are kept as in the base training. Rounds are chained: $\theta^{(\varrho+1)}$ is~(\ref{eq:sft_step}) applied to $\theta^{(\varrho)}$ for rounds $\varrho=0,\dots,n_{\rm r}-1$, and the final round's model is the result (Algorithm~\ref{alg:loop}).

\paragraph{What a round does.}
A round is rejection-sampling fine-tuning \citep{yuan2023rft,singh2023restem,zelikman2022star} with the RL policy as the sampler: the accepted responses of a prompt are samples from the policy conditioned on success, and fitting the model to them raises the likelihood of successful responses under the model's own training objective. Two properties of~(\ref{eq:sft}) are the design reasons for using it rather than a sampler-specific loss: it is the objective the model was built with, so a round is ordinary fine-tuning, and it trains the denoiser at all times $t$, which is the mechanism by which the improvement reaches every sampler rather than only the one that produced the samples. A round gives no direct signal on a prompt with no accepted response, so its coverage, the fraction of the pool with at least one success, bounds what the round can teach directly. This is where the critic enters.

\paragraph{Critic-guided sampling.}
The 2D critic is a function of partial states, so it can rank and improve candidates before they are decoded. Both operations use the per-position readings $r_{k,i}$ of~(\ref{eq:value2}) through the score
\begin{equation}
R_{\min}(s_k)=\min_{i\in\mathcal G}r_{k,i},
\label{eq:minscore}
\end{equation}
the most pessimistic of the $|\mathcal G|$ readings of the outcome. The minimum is used here, and the mean in training, because the two tasks differ: the actor improves the policy on average and needs a gradient at every position, whereas a single candidate is accepted only if the whole response is right, so what limits it is its least confident reading, and a gradient of the minimum acts on that reading. \emph{Pruning:} for each prompt, $n_{\rm g}$ trajectories are run in parallel to a step $k_p$; the trajectory with the largest $R_{\min}(s_{k_p})$ continues alone. \emph{Guidance:} at every stochastic step $k=1,\dots,K-1$ (it needs $\gamma>0$), before the denoiser call, the re-noised embeddings are moved along the gradient of the score on the generated positions,
\begin{equation}
\rvz_{t_k}\ \leftarrow\ \rvz_{t_k}+\kappa\,\psi_kc_\sigma\sqrt L\ \frac{\nabla_{\rvz}R_{\min}(s_k)}{\|\nabla_{\rvz}R_{\min}(s_k)\|},
\label{eq:guide}
\end{equation}
a step whose length is proportional to the noise scale $\psi_kc_\sigma$ of that step, a fraction $\kappa\sqrt{L/(|\mathcal G|d)}$ of the norm $\psi_kc_\sigma\sqrt{|\mathcal G|d}$ of the noise the sampler injects there, so that the pushed state stays within the spread of the sampler. The gradient with respect to the state in~(\ref{eq:guide}) is the same kind of quantity as the critic's slope that the actor multiplies with the denoiser Jacobian in~(\ref{eq:pwgrad}); guidance applies it to the sample at inference instead of to the parameters at training. A guided response costs $n_{\rm g}$ trajectories up to step $k_p$, one afterwards, and one backward pass through the critic per step; guided pools are therefore drawn with a cheap sampler (small $K$ and $\gamma$), and the search is meant to compensate the cheap sampler's lower accuracy. Ranking states of a sampler other than the critic's training sampler is an assumption of this use (Appendix~\ref{app:joint_details}). Verification and fine-tuning are unchanged.

\begin{algorithm}[H]
\caption{Amortization loop (with optional critic-guided sampling)}
\label{alg:loop}
\small
\begin{algorithmic}[1]
\REQUIRE policy $\theta^{(0)}$; optional critic $V_\phi$ of \methodSVGTwoD{}; verifier $R$; prompt pool $\mathcal P$; per round: a sampler and whether guidance is used; responses per prompt $n$; cap $M$; rounds $n_{\rm r}$; guidance constants $(n_{\rm g},k_p,\kappa)$
\FOR{round $\varrho=0,\dots,n_{\rm r}-1$}
  \FOR{each prompt $\rvc\in\mathcal P$ and $j=1,\dots,n$}
    \STATE Run the round's sampler under $\theta^{(\varrho)}$; with guidance, apply~(\ref{eq:guide}) at every step and prune $n_{\rm g}\to1$ at step $k_p$ by~(\ref{eq:minscore}); decode $\rvw^{(j)}$
  \ENDFOR
  \STATE $\mathcal A\leftarrow$ the verified, distinct responses, at most $M$ per prompt \hfill (eq.~\ref{eq:accept})
  \STATE $\theta^{(\varrho+1)}\leftarrow$ descend $\mathcal L_{\rm SFT}(\vartheta;\mathcal A)$ from $\vartheta=\theta^{(\varrho)}$ for the round's epochs \hfill (eq.~\ref{eq:sft_step})
\ENDFOR
\STATE \textbf{return} $\theta^{(n_{\rm r})}$
\end{algorithmic}
\end{algorithm}

\section{Experimental details}
\label{app:details}

\paragraph{Base models.} The released-model rows of Table~\ref{tab:main} evaluate each released checkpoint with the task prompt and no task training: GPT-2 (124M), LLaDA-8B-Instruct (256 generated tokens), and ELF-B's OpenWebText checkpoint under the deployment sampler.
ELF-B \citep{hu2026elf}: 105M-parameter transformer denoiser, frozen T5-small \citep{raffel2020t5} encoder embeddings ($d{=}512$), x-prediction with self-conditioning and training-time classifier-free guidance, PyTorch port of the released code. A denoiser call $\mathrm d_\theta$ combines a prompt-conditioned and an unconditioned evaluation, $\rvx_{t_k}=\rvx_{\rm u}+\omega(\rvx_{\rm c}-\rvx_{\rm u})$, with guidance weight $\omega=2$ in every rollout and evaluation (Appendix~\ref{app:datasets}); the critic's copy is evaluated without guidance. Supervised fine-tuning uses the released recipe (AdamW, learning rate $2\times10^{-3}$ with warm-up, global batch 256--512, EMA of weights; EMA decay 0.99 for the short Game24 run and 0.999 otherwise). The Countdown base is epoch 60 of a 100-epoch run; Sudoku uses the converged model (76 epochs on 400k puzzles); K\&K the epoch-4 checkpoint of a 40-epoch schedule stopped at epoch 9; Game24 the epoch-20 checkpoint of a 40-epoch run initialised from the Countdown model.

\paragraph{Critic.}
The critic reuses the denoiser as a feature extractor: a copy of the whole network, fine-tuned during pre-training (learning rate $10^{-5}$) and frozen during RL, reads the state $s_k=(\rvz_{t_k},\rvx_{t_{k-1}})$ at the time of the corresponding denoiser call, exactly as the sampler's call does, and returns the position features $\rvh=(h_1,\dots,h_L)$ of Section~\ref{sec:critic}. The MLP head $f_\phi$ ($d{+}1\to256\to64\to1$, GELU; the extra input is the time) produces the per-position readings $r_{k,i}$ of~(\ref{eq:value2}), pooled over the generated positions. The 1D critic is the same network with the head applied to the mean of the features (eq.~\ref{eq:value1}), pre-trained on the same rollouts and stored steps with the same epoch budget and selection (kept epochs 10, 3, 10 and 6 on Countdown, Sudoku, K\&K and Game24; the 2D heads keep epochs 9, 1, 10 and 6) and trained jointly with the same settings; the two rows of Table~\ref{tab:main} differ in nothing else. Pre-training: rollouts of the task-SFT policy with states stored at several denoising steps and the verifier reward, the BCE loss~(\ref{eq:critic_loss}) or~(\ref{eq:critic2d_loss}), batch 64, and the epoch with the best held-out field score kept (Table~\ref{tab:critic_pretrain}); the field score is the mean over the stored steps of half the Pearson correlation between the minimum per-position reading and the reward plus half the within-prompt rank correlation between the pooled value and the reward (for the 1D critic, which has no per-position readings, the pooled value replaces the minimum). The states are sampled under the deployment sampler except on Countdown, where states of the deployment sampler at the early and middle steps carry almost no information about the outcome (held-out AUC about $0.6$ for $k\le24$ of 64) and the critic is pre-trained on states of the cheap $\gamma4$@16 sampler instead; the choice is made by the held-out field score.
\begin{table}[H]\centering\small
\caption{Critic pre-training per task (Algorithm~\ref{alg:jpc}, Require).}\label{tab:critic_pretrain}
\begin{tabular}{@{}lllll@{}}\toprule
task & states from & rollouts & stored steps & epochs (kept)\\\midrule
Countdown & $\gamma4$@16 & 64k & 1, 2, 4, 6, 8, 10, 12, 14 of 16 & 12 (9)\\
Sudoku & $\gamma28$@32 & 32k & 1, 6, 12, 18, 24 of 32 & 12 (1)\\
K\&K & $\gamma56$@64 & 32k & 1, 8, 16, 24, 32, 40, 48 of 64 & 10 (10)\\
Game24 & $\gamma56$@64 & 16k & 1, 8, 16, 32, 48 of 64 & 6 (6)\\\bottomrule\end{tabular}\end{table}

\paragraph{\methodSVGScalar{} and \methodSVGTwoD{} (Algorithm~\ref{alg:jpc}).}
32 prompts per iteration, deployment-sampler rollouts without gradients, one random step in the window $\{8,\dots,40\}$ of 64 ($\{4,\dots,20\}$ of 32) recomputed with gradients, loss $-V_\phi(s_k)$ averaged over prompts, AdamW on the actor with gradient clipping at 1.0, learning rate $10^{-6}$ decayed on a cosine to $10^{-7}$ over the run (1500 iterations on Countdown and Game24, 1000 on Sudoku and K\&K). Replay buffer of 16{,}384 states, filled each iteration from 32 fresh prompts rolled out to the end of the trajectory (separately from the actor's rollouts) with the state at the top of the window ($k{=}40$ of 64; $20$ of 32) and its outcome; 2 critic steps per iteration on minibatches of 64, updating the head at learning rate $10^{-4}$ with the copy of the denoiser frozen, loss~(\ref{eq:critic2d_loss}); every 10 iterations the pooled value's AUC (area under the ROC curve) against the outcome is computed on the newest 256 buffer states, and if it is below $0.9$ the critic takes 8 extra steps. In the symbols of Algorithm~\ref{alg:jpc}: $T{=}1500$ (Countdown, Game24) or $1000$ (Sudoku, K\&K), $B{=}32$, $B'{=}64$, $N_{\mathcal B}{=}16{,}384$, $m_{\rm c}{=}2$, $m'_{\rm c}{=}8$, $n_q{=}10$, $N_q{=}256$, $a_0{=}0.9$, $\mathcal W=\{8,\dots,40\}$ of 64 ($\{4,\dots,20\}$ of 32), $\eta_\theta=10^{-6}\to10^{-7}$, $\eta_\phi=10^{-4}$. Checkpoints are saved every 250 iterations (\methodGRPO{}: every 100) and evaluated as described under \emph{Evaluation} below. The 1D critic has no per-position readings and minimises~(\ref{eq:critic_loss}). The cheap-sampler row of Table~\ref{tab:needs} and its reference use 1000 iterations at a constant learning rate of $10^{-6}$, a buffer of the states at steps $8,16,24,32,40$ of 64, the pooled loss term only and no AUC check. Runs take 45--90 minutes for 1000 iterations on one H200 (2--3 hours on an L40).

\paragraph{\methodGRPO{}.}
Group size $G{=}8$, 16 prompts per iteration (128 rollouts), $\gamma4$@16 rollouts, group-normalised advantage as in~(\ref{eq:grpo_loss}), no KL penalty, no ratio clipping (the clipped variants are in Table~\ref{tab:grpo_variants}), learning rate $3\times10^{-6}$, 600 iterations, checkpoints every 100 iterations, evaluated as under \emph{Evaluation} below. The sum over denoising steps in~(\ref{eq:grpo_loss}) is estimated on 8 transitions drawn uniformly at random per update, with uniform weight over steps and positions.

\paragraph{AR control.}
GPT-2 (124M) \citep{radford2019gpt2} is fine-tuned on each task's training set (batch 128, learning rate $3\times10^{-4}$, warm-up 500; 10 epochs, 20 on Sudoku), decoded greedily and scored by the same verifiers as the flow model: Countdown 44.2\% (epoch 8), Sudoku 1.2\% (23 of 2{,}000 puzzles at epoch 20), K\&K 92.9\% (epoch 10), Game24 26.9\% (epoch 2, initialised from the Countdown-fine-tuned weights, as the flow model is). Token-level GRPO/PPO uses 32 prompts $\times$ 8 samples per iteration, clip 0.2, no KL, 2 inner epochs, 600 iterations evaluated every 50, learning rate $10^{-5}$ on Countdown, Sudoku and Game24 and $3\times10^{-6}$ on K\&K, where the supervised policy is already sharp; dropout is disabled during RL, as in standard PPO implementations, so that the two forward passes that form the ratio agree; the log importance ratio is clamped to $\pm10$ as a numerical safeguard. Curves: Appendix~\ref{app:ar_curves}.

\paragraph{Amortization loop.}
The loop of Table~\ref{tab:loop} runs 4 rounds on Countdown and 5 on K\&K (Table~\ref{tab:loop_full} adds one Sudoku round; Game24 has no pool of unlabelled problems, since it is split by number set). In each round the current policy samples its pool at the deployment sampler (Countdown: a disjoint unlabelled pool of 50k problems, $n{=}4$ samples per problem; Game24, for the check only: the 4{,}693 other orderings of the training sets, $n{=}8$; K\&K: the 50k training problems in round 1 and a fresh 50k pool in every later round, $n{=}2$; Sudoku: 50k fresh puzzles, $n{=}2$), verifier-accepted samples are kept with at most four (Countdown, Game24) or one (unique-solution tasks) per problem, and the policy is fine-tuned on them with the supervised recipe (20 epochs on Countdown, 8 elsewhere) from the previous round's final checkpoint; the tables report the final checkpoint of the last round. Starting policies: a 2D-critic checkpoint from the runs of Table~\ref{tab:main}, the SFT base, and a \methodGRPO{} checkpoint. The gold-label control starts from the 2D-critic policy and fine-tunes, with the same rounds and schedule, on the gold labels of the same pools instead of the policy's samples (one gold derivation per problem; on K\&K the round-1 pool is the 50k training problems with their labels). On Countdown, rounds 3--4 of the 2D-policy loop sample the pool with the critic-guided sampler of Appendix~\ref{app:loop}: the cheap sampler $\gamma4$@16, $n_{\rm g}{=}4$ trajectories per prompt pruned at $k_p{=}8$ by the score~(\ref{eq:minscore}), and the push~(\ref{eq:guide}) with $\kappa{=}5$ at every step, all with the jointly trained 2D critic of the starting policy; this restores the pool's coverage (Table~\ref{tab:loop_full}). The controls have no critic and sample plainly in every round.

\paragraph{Evaluation.}\label{app:seeds}
Every number is one evaluation of the held-out prompts (2{,}000 per task, 1{,}601 for Game24) at classifier-free guidance 2. The held-out prompts are split by index into a validation half (even indices) and a test half (odd indices): checkpoints are selected by their validation accuracy at the deployment sampler (\methodGRPO{}: every 100 of 600 iterations; the SVG variants: every 250; LLaDA: every 960 inner steps; ties to the earlier checkpoint) and every reported accuracy is measured on the test half.

\paragraph{Rollout budgets.}
\methodGRPO{} uses 16 prompts $\times$ 8 samples $=128$ rollouts per iteration under $\gamma4$@16 for 600 iterations. The SVG variants (Algorithm~\ref{alg:jpc}) use 32 actor rollouts and 32 buffer rollouts per iteration under the deployment sampler for 1{,}000--1{,}500 iterations. GPT-2's token-level GRPO/PPO uses 32 prompts $\times$ 8 samples per iteration for 600 iterations. We compare computational cost using total training FLOPs, as described below.

\paragraph{Compute.}
The campaign used about 300 GPU-hours on A100 (40/80 GB), H200 and L40 nodes; a single pathwise run costs about one GPU-hour on an H200, a loop round 15--40 minutes. We measure total training FLOPs across all network evaluations, including rollouts, actor updates, critic pre-training, and online critic updates. Following the standard parameter-based FLOPs approximation \citep{hoffmann2022training}, we count a forward pass as one unit and a full backward pass as two units. We additionally approximate an input-gradient-only backward pass as one unit, since the critic backbone is frozen during RL.

For ELF-B, Table~\ref{tab:matched_compute} reports total training compute relative to the task-SFT base. In the order Countdown, Sudoku, K\&K, and Game24, both SVG variants use 263, 269, 277, and 232 PFLOPs, including critic pre-training and online critic computation. The $\gamma4$@16 GRR runs use 97, 221, 130, and 97 PFLOPs; the GRR runs trained with the deployment sampler use 208, 306, 278, and 208 PFLOPs. Here one PFLOP denotes $10^{15}$ floating-point operations.

\section{Additional results}
\label{app:results}

\subsection{Main result: supporting results}
the learning curves of the autoregressive reference (Appendix~\ref{app:ar_curves}), \methodGRPO{} read under both samplers (Appendix~\ref{app:grpo_chain}), and the recipe of the masked-diffusion reference (Appendix~\ref{app:d1}).

\subsubsection{Autoregressive reference: learning curves}
\label{app:ar_curves}
Table~\ref{tab:ar_curves} gives the curves behind the GPT-2 rows of Table~\ref{tab:main}: the supervised stage by epoch and token-level GRPO/PPO by iteration. Token-level RL climbs steadily on Countdown and Game24, much of it within the first 100 iterations; on K\&K, where the supervised policy is already sharp, it holds; on Sudoku, where the supervised model solves about one puzzle in a hundred, the reward is too sparse to move it. The autoregressive policy sees the same problems and verifier as the flow policies and at least as many rollouts per iteration; what differs is the action space the estimator acts on.

\begin{table}[H]\centering\small
\caption{AR control (GPT-2, 124M), held-out accuracy (\%) with greedy decoding. Top: supervised fine-tuning by epoch (Sudoku: epochs 2, 4, \dots, 20, scored by the SFT script's solved-grid metric; the verifier gives 1.2 at epoch 20). Bottom: token-level GRPO/PPO by iteration (32 prompts $\times$ 8 samples per iteration). Bold: the highest cell of each row.}\label{tab:ar_curves}
\resizebox{\linewidth}{!}{%
\begin{tabular}{@{}lcccccccccc@{}}\toprule
SFT epoch & 1 & 2 & 3 & 4 & 5 & 6 & 7 & 8 & 9 & 10\\\midrule
Countdown & 26.5 & 30.8 & 38.1 & 38.8 & 41.1 & 43.9 & 42.4 & \textbf{44.2} & 43.6 & 42.8\\
Sudoku (epochs 2--20) & \phantom{0}0.6 & \phantom{0}0.9 & \phantom{0}1.1 & \phantom{0}1.3 & \phantom{0}1.3 & \phantom{0}1.4 & \phantom{0}1.4 & \phantom{0}1.4 & \textbf{\phantom{0}1.7} & \phantom{0}1.7\\
K\&K & 28.2 & 46.6 & 68.7 & 82.8 & 87.1 & 88.5 & 91.4 & 92.3 & 92.8 & \textbf{92.9}\\
Game24 (init.\ from the Countdown model) & 19.6 & \textbf{26.9} & 22.1 & 21.3 & 19.9 & 21.7 & 21.0 & 20.8 & 20.7 & 21.6\\
Game24 (init.\ from vanilla GPT-2) & \phantom{0}1.1 & \phantom{0}9.8 & 11.4 & \textbf{12.4} & \phantom{0}9.0 & \phantom{0}9.5 & 10.3 & 11.2 & \phantom{0}9.7 & \phantom{0}9.7\\
\midrule
RL iteration & 0 & 100 & 150 & 200 & 300 & 400 & 500 & 550 & 600 & \\\midrule
Countdown, $10^{-5}$ & 44.2 & 52.3 & 55.2 & 54.8 & 57.8 & 59.0 & 60.6 & 61.3 & \textbf{61.7} & \\
Sudoku, $10^{-5}$ & \phantom{0}1.2 & \phantom{0}1.1 & \phantom{0}1.4 & \textbf{\phantom{0}1.6} & \phantom{0}1.6 & \phantom{0}1.5 & \phantom{0}1.6 & \phantom{0}1.4 & \phantom{0}1.6 & \\
K\&K, $3\times10^{-6}$ & 92.9 & 93.0 & \textbf{93.2} & 93.1 & 92.9 & 92.8 & 92.7 & 92.8 & 92.7 & \\
Game24, $10^{-5}$ & 26.9 & 39.8 & 40.2 & 42.5 & 45.8 & 44.2 & 42.8 & \textbf{46.2} & 46.0 & \\
\bottomrule\end{tabular}}\end{table}

\subsubsection{\methodGRPO{} under its own sampler and under the deployment sampler}
\label{app:grpo_chain}
\methodGRPO{} is trained under the 16-step $\gamma4$@16 sampler (the cheap sampler except on Sudoku) and, like every row of Table~\ref{tab:main}, read under the deployment sampler. Two checks show that this mismatch is not what limits it. Read under its own sampler, the Countdown runs gain about one point over the base. Trained under the deployment sampler instead, with the same recipe otherwise, \methodGRPO{} gains $+0.1$ on Countdown and $+3.5$ on Game24. The estimator gains little under either sampler.

\subsubsection{The masked-diffusion reference: d1 on our tasks}
\label{app:d1}
d1 \citep{zhao2025d1} is the most used RL recipe for masked diffusion language models: LLaDA-8B-Instruct \citep{nie2025llada} with LoRA, a supervised stage, and diffu-GRPO, a GRPO variant whose per-token log-probabilities come from one forward pass in which all completion tokens are masked. That estimator requires a mask token and a discrete masked-diffusion model, so it cannot be applied to a continuous flow model; \methodGRPO{} is its counterpart on ELF. We ran d1's own recipe, unchanged, on its own base model and on our four held-out sets, with the authors' code and hyperparameters (LoRA rank 128 on all attention and MLP projections, six generations per prompt, 256-token completions with 128 diffusion steps, block length 32, learning rate $3\times10^{-6}$, 12 inner iterations, KL weight 0.04, clip $0.5$), prompts and rewards in d1's own formats, and accuracy as our exact-success metric. The supervised stage is completion-only masked SFT on the same training data as ELF-B's task SFT (3--20 epochs per task); diffu-GRPO then runs from the merged SFT model for 8{,}000 inner steps, about 32k rollouts, the budget of our pathwise runs. The comparison is across architectures and training histories: LLaDA-8B-Instruct was never trained on these tasks and its supervised stage is a LoRA run, whereas ELF-B is trained in full, so we make no claim about the ceiling of masked diffusion models. What Table~\ref{tab:main} shows is the behaviour of the estimator: diffu-GRPO moves the supervised model by under three points on every task, consistent with d1's own results on its easier three-number Countdown and $4{\times}4$ Sudoku.

\subsubsection{Comparison at matched compute}
\label{app:matched_compute}
Table~\ref{tab:matched_compute} gives an approximately compute-matched version of Table~\ref{tab:main}. With comparable training compute, SVG-2D still outperforms GRR on all four datasets, supporting that its performance gains are not simply due to a larger compute budget.

\begin{table}[H]
\centering\small
\arrayrulecolor{black}
\caption{Test accuracy (\%) and total training compute (PFLOPs). The GRR 64-sampler row is newly evaluated; all other accuracies are from Table~\ref{tab:main}. Critic-stage costs are included in the SVG totals. Compute is measured relative to the task-SFT base.}
\label{tab:matched_compute}
\setlength{\tabcolsep}{3pt}
\resizebox{\linewidth}{!}{%
\begin{tabular}{@{}lrrrrrrrr@{}}
\toprule
\textbf{Method} & \multicolumn{2}{c}{\textbf{Countdown}} & \multicolumn{2}{c}{\textbf{Sudoku}} & \multicolumn{2}{c}{\textbf{K\&K}} & \multicolumn{2}{c}{\textbf{Game24}} \\
\cmidrule(lr){2-3}\cmidrule(lr){4-5}\cmidrule(lr){6-7}\cmidrule(lr){8-9}
 & Acc. & PFLOPs & Acc. & PFLOPs & Acc. & PFLOPs & Acc. & PFLOPs \\
\midrule
Task-SFT base & 42.8 & 0 & 64.0 & 0 & 68.2 & 0 & 44.0 & 0 \\
\methodGRPO{} ($\gamma4$@16) & 42.6 & 97 & 67.9 & 221 & 71.4 & 130 & 47.4 & 97 \\
\shortstack[l]{\methodGRPO{} ($\gamma56$@64;\\$\gamma28$@32 Sudoku), new} & 42.8 & 208 & 64.2 & 306 & 77.0 & 278 & 47.3 & 208 \\
\methodSVGScalar{} & 47.4 & 263 & 64.1 & 269 & 76.5 & 277 & 51.2 & 232 \\
\textbf{\methodSVGTwoD{}} & \textbf{49.7} & 263 & \textbf{68.7} & 269 & \textbf{77.5} & 277 & \textbf{58.2} & 232 \\
\midrule
Critic stage (included in SVG) & -- & 67 & -- & 113 & -- & 103 & -- & 36 \\
\bottomrule
\end{tabular}}
\arrayrulecolor{black}
\end{table}

\subsection{The estimator}
This group supports the estimator paragraph of Section~\ref{sec:experiments}:  the parts of GRPO put back one at a time (Appendix~\ref{app:grpo_variants}), the two estimators on the same rollouts (Appendix~\ref{app:snr}), and one step along each gradient (Appendix~\ref{app:onestep}).

\subsubsection{GRPO with ratio, clipping and anchor}
\label{app:grpo_variants}
\methodGRPO{} takes one optimizer step per batch of rollouts and has no KL anchor; at the parameters that drew the batch every importance ratio equals 1, so the first step of clipped GRPO with $\beta=0$ coincides with a step on~(\ref{eq:grpo_loss}). Table~\ref{tab:grpo_variants} puts the parts of the original GRPO objective back, on the rollout budget of Table~\ref{tab:main}: two steps per batch with the ratio taken per denoising step and position (which the diagonal covariance of~(\ref{eq:kernel}) allows), then the clip at $\epsilon=0.2$, then the anchor at $\beta=0.04$ on the closed-form KL of the Gaussian transitions. The second step and the clip add a few points on Game24 and K\&K and under a point on Countdown; the anchor at its usual strength takes the clipped two-step form back to about the level of \methodGRPO{}. With either row the order of Table~\ref{tab:main} is unchanged: the 2D critic stays ahead of the clipped two-step form on every task.

\begin{table}[H]\centering\small
\caption{GRPO variants on the flow model, held-out accuracy (\%) under the deployment sampler. All rows draw the same rollouts per iteration for the same number of iterations as the \methodGRPO{} row of Table~\ref{tab:main}.}\label{tab:grpo_variants}
\begin{tabular*}{\linewidth}{@{\extracolsep{\fill}}lcccccc@{}}\toprule
 & steps per batch & ratio & clip & anchor & Countdown & Game24 \\\midrule
base & & & & & 42.2 & 45.0 \\
\methodGRPO{} & 1 & --- & --- & --- & 43.1 & 46.8 \\
ratio only & 2 & yes & --- & --- & 42.6 & 50.1 \\
ratio and clip & 2 & yes & 0.2 & --- & 44.7 & 51.1 \\
GRPO, ratio, clip and anchor & 2 & yes & 0.2 & 0.04 & 43.2 & 46.0 \\\bottomrule
\end{tabular*}

\vspace{8pt}
{\setlength{\tabcolsep}{3pt}\begin{tabular*}{\linewidth}{@{\extracolsep{\fill}}lcccc@{}}\toprule
 & Countdown & Sudoku & K\&K & Game24 \\\midrule
\methodGRPO{} & 43.1 & 69.7 & 71.6 & 47.5 \\
ratio and clip, two steps & 43.9 & 68.0 & 74.1 & 50.0 \\\bottomrule
\end{tabular*}}\end{table}

\subsubsection{The two estimators on the same rollouts: SNR and half-split cosine}
\label{app:snr}
The top right panel of Table~\ref{tab:main} asks, on one set of Game24 rollouts of the base under the deployment sampler and with one differentiable denoiser call as in Algorithm~\ref{alg:jpc}, how well the mean gradients of two random halves of the rollouts agree. The ruler is the $A{=}1$ control, whose true mean is zero. The flow model's score-function gradient is indistinguishable from it under every form of the advantage up to $2{,}048$ rollouts; GPT-2's token-level gradient looks the same at $256$ rollouts and then separates as $N$ grows, the signature of a weak but genuine per-sample signal (the estimator is the same, the action space is not); the pathwise gradient through the same denoiser call agrees with itself already at $256$ rollouts, with an SNR nearly three times the floor. Table~\ref{tab:snr} gives the numbers; the derivation of the two gradients is in Appendix~\ref{app:variance}.

\begin{table}[!t]
\centering\small
\caption{The score-function~(\ref{eq:sf}) and pathwise~(\ref{eq:pwgrad}) estimators on Game24 rollouts of the task-SFT policy at denoising step $k{=}24$ of 64 (256 rollouts for the SNR, up to 2{,}048 for the half-split cosine; pathwise term with the pre-trained 2D critic). SNR: norm of the mean per-sample gradient over its root-mean-square deviation; a zero-mean estimator sits at $1/\sqrt N=0.0625$, the $A{=}1$ row. Half-split cosine: agreement between the mean gradients of two disjoint halves of the rollouts (mean$\pm$sd over 10 splits) as the number of rollouts grows; a genuine signal rises with $N$, a zero-mean estimator stays at $0$. Last two rows: the same statistic for GPT-2's token-level GRPO on Game24 (the panel's line) and on Countdown. Bold: the pathwise estimator.}
\label{tab:snr}
\setlength{\tabcolsep}{4pt}
\begin{tabular}{@{}lc|ccc@{}}
\toprule
 & SNR, $N{=}256$ & \multicolumn{3}{c}{half-split cosine} \\
 & & $N{=}256$ & $N{=}1024$ & $N{=}2048$ \\
\midrule
score-function, mean-centred $A$ & 0.061 & $-0.03\pm0.03$ & $-0.04\pm0.03$ & $0.00\pm0.04$ \\
score-function, GRPO $A$ (groups of 4) & 0.065 & $0.03\pm0.03$ & $0.02\pm0.03$ & $-0.03\pm0.03$ \\
score-function, $A{=}1$ (zero-mean control) & 0.063 & $0.02\pm0.02$ & $0.04\pm0.02$ & $0.02\pm0.03$ \\
pathwise, 2D critic & \textbf{0.177} & $\mathbf{0.79\pm0.03}$ & $\mathbf{0.94\pm0.01}$ & $\mathbf{0.97\pm0.00}$ \\
\midrule
GPT-2, token-level GRPO (Game24) & --- & $0.09\pm0.05$ & $0.35\pm0.04$ & $0.34\pm0.07$ \\
GPT-2, token-level GRPO (Countdown) & --- & $0.05\pm0.05$ & $0.17\pm0.05$ & $0.27\pm0.02$ \\
\bottomrule
\end{tabular}
\end{table}

\subsubsection{One step along each gradient}
\label{app:onestep}
The half-split cosine says that the pathwise gradient is reproducible, not that it points the right way. The lower right panel of Table~\ref{tab:main} tests the direction by intervention at iteration 0 of RL, with the task-SFT base and the pre-trained 2D critic and no training. We draw $1{,}024$ rollouts of $128$ training prompts under the deployment sampler and compute, on the same rollouts and the same four denoising steps of the window per rollout, the mean score-function gradient of~(\ref{eq:grpo_loss}) and the mean pathwise gradient of $V_\phi(s_k)$. The policy is then moved once, $\theta'=\theta+\eta\,g/\lVert g\rVert$, and evaluated on the held-out prompts under the deployment sampler with the sampler's random stream fixed, so that every $\theta'$ sees the same prompts and the same noise as the base. The length $\eta$ is measured against $\Lambda$, the parameter distance that a 2D-critic run of Table~\ref{tab:main} travels from the base; each gradient is applied as AdamW applies a first step (global clipping at 1, then $g/(|g|+\epsilon)$ per coordinate) and, for comparison, raw; a Gaussian random direction is the control. Table~\ref{tab:onestep} gives the paired change in accuracy. In the form the optimizer uses, one pathwise step gains about four points on Game24 at a quarter to a half of $\Lambda$ and three to four on Countdown from a half of $\Lambda$ on, with intervals that exclude zero; the score-function step and the random step stay within about a point of the base at every length. The raw gradients are a different matter: over $0.9$ of the squared norm of either gradient lies in one percent of the coordinates, so a raw step of any useful length moves a few coordinates far outside the region where the critic was fitted and the raw pathwise step destroys the policy; the per-coordinate normalisation of AdamW is part of what makes the pathwise update work. On Game24 the gain turns over at $\eta{=}\Lambda$: a single direction from a critic that is not refit cannot be followed as far as training goes.

\begin{table}[!t]\centering\small
\caption{One step from the base along each mean gradient of $1{,}024$ rollouts. Each cell is the paired change in held-out accuracy (points) against the base under the deployment sampler with the sampling noise fixed; brackets: $95\%$ bootstrap interval over prompts. \emph{AdamW}: the direction of a first AdamW step of that gradient; \emph{raw}: the gradient itself; all steps have length $\eta$.}\label{tab:onestep}
\begin{tabular}{@{}lccccc@{}}\toprule
$\eta/\Lambda$ & pathwise, AdamW & score-function, AdamW & random & pathwise, raw & score-function, raw \\\midrule
\multicolumn{6}{@{}l}{\emph{Game24} (base $43.0$, $\Lambda=1.30$ over $750$ iterations, 1601 prompts)}\\
$1/256$ & $+0.0$ {\scriptsize[-0.5, +0.5]} & $+0.2$ {\scriptsize[-0.2, +0.7]} & $+0.5$ & $+0.5$ & $+0.1$ \\
$1/128$ & $+0.1$ {\scriptsize[-0.5, +0.7]} & $+0.3$ {\scriptsize[-0.2, +0.9]} & $+0.3$ & $-0.2$ & $+0.1$ \\
$1/64$ & $+0.0$ {\scriptsize[-0.7, +0.7]} & $+0.2$ {\scriptsize[-0.2, +0.7]} & $+0.7$ & $-0.6$ & $+0.3$ \\
$1/32$ & $+0.9$ {\scriptsize[+0.0, +1.9]} & $+0.8$ {\scriptsize[+0.2, +1.5]} & $+0.0$ & $-1.3$ & $+0.4$ \\
$1/16$ & $+1.0$ {\scriptsize[-0.2, +2.1]} & $+0.9$ {\scriptsize[+0.2, +1.7]} & $+0.4$ & $-2.1$ & $+0.4$ \\
$1/4$ & $+4.2$ {\scriptsize[+2.4, +5.9]} & $+1.2$ {\scriptsize[-0.1, +2.5]} & $+0.4$ & $-13.1$ & $+0.3$ \\
$1/2$ & $+4.4$ {\scriptsize[+2.2, +6.6]} & $+1.6$ {\scriptsize[-0.1, +3.4]} & $+0.2$ & $-43.0$ & $-2.1$ \\
$1$ & $-1.2$ {\scriptsize[-3.7, +1.3]} & $+0.1$ {\scriptsize[-1.9, +2.2]} & $+0.2$ & $-43.0$ & $-2.1$ \\
\midrule
\multicolumn{6}{@{}l}{\emph{Countdown} (base $42.2$, $\Lambda=1.39$ over $1500$ iterations, 2000 prompts)}\\
$1/256$ & $-0.1$ {\scriptsize[-0.9, +0.5]} & $+0.1$ {\scriptsize[-0.7, +0.7]} & $+0.2$ & $-0.5$ & $-0.2$ \\
$1/128$ & $+0.1$ {\scriptsize[-0.6, +0.9]} & $-0.1$ {\scriptsize[-0.9, +0.5]} & $+0.3$ & $+0.1$ & $+0.2$ \\
$1/64$ & $-0.1$ {\scriptsize[-0.8, +0.7]} & $-0.1$ {\scriptsize[-0.8, +0.6]} & $+0.3$ & $-0.3$ & $-0.1$ \\
$1/32$ & $+0.2$ {\scriptsize[-0.5, +1.1]} & $-0.1$ {\scriptsize[-1.0, +0.7]} & $+0.4$ & $-0.9$ & $+0.2$ \\
$1/16$ & $+0.7$ {\scriptsize[-0.4, +1.7]} & $+0.4$ {\scriptsize[-0.8, +1.5]} & $+0.6$ & $-0.9$ & $+0.0$ \\
$1/4$ & $+2.6$ {\scriptsize[+1.0, +4.3]} & $-0.1$ {\scriptsize[-1.7, +1.4]} & $-0.4$ & $-6.7$ & $-0.4$ \\
$1/2$ & $+3.3$ {\scriptsize[+1.4, +5.3]} & $+0.2$ {\scriptsize[-1.8, +2.1]} & $+0.0$ & $-36.7$ & $-0.7$ \\
$1$ & $+3.9$ {\scriptsize[+1.6, +6.2]} & $-2.1$ {\scriptsize[-4.5, +0.1]} & $-0.4$ & $-42.2$ & $-4.1$ \\
\bottomrule\end{tabular}\end{table}

\subsection{Position sensitivity of the token-wise critic}
\label{app:position_sensitivity}
Supervising each token-wise critic output with the same sequence-level reward does not guarantee semantic token-level credit assignment. We examine this distinction through two experiments: whether the 2D critic assigns lower readings to incorrect token positions, and whether its token-wise readings respond to targeted perturbations.

\paragraph{Correlation with incorrect positions.}
We evaluate the 2D critic on failed Sudoku solutions, where each output position corresponds directly to an answer cell and can therefore be identified as correct or incorrect. For each failed response, we take the single token with the lowest reading and record 1 if that token is actually wrong and 0 otherwise; Precision is the average over all failed responses. Table~\ref{tab:critic_position_precision} reports the range across the $K$ steps. Precision is 0.62--0.65, compared with 0.42 for a randomly selected position. The readings thus contain information about which positions are incorrect despite the sequence-level training supervision.

\begin{table}[H]
\centering\small
\arrayrulecolor{black}
\caption{Identifying incorrect positions in failed Sudoku solutions. Precision measures whether the position with the lowest critic reading is incorrect; the reported range is over the $K$ denoising steps.}
\label{tab:critic_position_precision}
\begin{tabular}{@{}lc@{}}
\toprule
\textbf{Metric} & \textbf{Range over $K$ steps} \\
\midrule
Precision & \textbf{0.62--0.65} \\
Random baseline precision & 0.42 \\
\bottomrule
\end{tabular}
\arrayrulecolor{black}
\end{table}

\paragraph{Counterfactual token corruption.}
Starting from a correct solution, we corrupt a single token and measure the change in the critic reading at every position. We report the percentage of cases in which the corrupted position exhibits the largest decrease in reading among all positions, together with the random-position baseline (Table~\ref{tab:critic_token_corruption}). Across all four tasks, the corrupted position is substantially more likely to exhibit the largest decrease than expected by chance. This response provides evidence that the critic is sensitive to localized token perturbations and captures position-specific information despite being trained with sequence-level supervision.

\begin{table}[H]
\centering\small
\arrayrulecolor{black}
\caption{Response to single-token corruption of correct solutions. The middle column reports the percentage of cases in which the corrupted position has the largest decrease in critic reading.}
\label{tab:critic_token_corruption}
\begin{tabular}{@{}lcc@{}}
\toprule
\textbf{Task} & \shortstack{\textbf{Corrupted position with}\\\textbf{largest reading decrease}} & \textbf{Random baseline} \\
\midrule
Sudoku & 11--22\% & 1.2\% \\
K\&K & 23--30\% & 6.5\% \\
Countdown & 8--61\% & 3.7\% \\
Game24 & 13--26\% & 5.6\% \\
\bottomrule
\end{tabular}
\arrayrulecolor{black}
\end{table}

These experiments show that the critic captures position-specific information and responds to localized corruption, which may contribute to the 2D critic's advantage. An incorrect answer token can be sufficient to establish sequence-level failure, whereas a correct token alone cannot establish success. This asymmetry motivates position-sensitive critic readings even under sequence-level supervision. They do not establish that semantic token-level credit assignment is the sole source of the 2D critic's advantage: gradient scaling, conditioning, and optimization dynamics may also contribute to its performance gains. We therefore distinguish position-sensitive reward estimation from fully interpretable token-level credit assignment.

\subsection{The critic and the joint training}
This group supports the critic paragraphs of Section~\ref{sec:experiments}: how the critic's value develops along the trajectory and over positions (Appendices~\ref{app:critic_quality} and~\ref{app:critic_maps}), the two critics over training on all tasks (Appendix~\ref{app:traj}), the measurements showing that they are interchangeable at a fixed policy (Appendix~\ref{app:critic_equiv}), the frozen critic (Appendix~\ref{app:frozen}), and the training dynamics of the reference runs (Appendix~\ref{app:dynamics}).

\subsubsection{Critic quality along the trajectory}
\label{app:critic_quality}
Table~\ref{tab:critic_quality} reports, for the pre-trained 2D critic of each task, how well the pooled value $V_\phi$ of~(\ref{eq:value2}) at a \emph{noisy} state at step $k$ predicts the final reward of the trajectory it lies on: the Pearson correlation between $V_\phi$ and the reward over held-out rollouts, and pick@$V_\phi$, the top-1 accuracy of selecting a successful sample among a group of 8 rollouts of the same prompt by ranking on $V_\phi$ (only mixed groups count; the random baseline is the success rate of a random pick in such groups). The profile is the same on every task: near chance at $k{=}1$, informative from roughly a quarter of the trajectory, and essentially an oracle in the last third. This is the token-commitment profile of the sampler (Appendix~\ref{app:commit}); it is the reason the pathwise step is drawn from the middle of the trajectory rather than from its start, where the critic has no gradient to give, or its end, where the state is already committed. For Countdown and Game24 the per-position AUC of $r_{k,i}$ against oracle step labels is also reported; the K\&K and Sudoku verifiers do not produce per-position labels.

\begin{table}[H]\centering\small
\caption{Pre-trained 2D critics on held-out rollouts of the base policy (deployment sampler; Countdown critic trained on $\gamma4$@16 states). pick@$V_\phi$ over groups of 8 (random baseline in the header); Pearson correlation of the pooled value with the 0/1 reward.}\label{tab:critic_quality}
\begin{tabular}{@{}llccccc@{}}\toprule
task (random) & metric & \multicolumn{5}{c}{flow step $k$}\\\midrule
Countdown (0.32), 16 steps & $k$ & 1 & 4 & 8 & 14 & \\
& pick@$V_\phi$ & 0.35 & 0.56 & 0.97 & 1.00 & \\
& Pearson & 0.14 & 0.33 & 0.91 & 0.98 & \\
& token AUC & 0.22 & 0.27 & 0.46 & 0.47 & \\\midrule
Sudoku (0.53), 32 steps & $k$ & 1 & 8 & 16 & 24 & \\
& pick@$V_\phi$ & 0.55 & 0.89 & 0.96 & 0.98 & \\
& Pearson & 0.48 & 0.85 & 0.92 & 0.97 & \\\midrule
K\&K (0.53), 64 steps & $k$ & 1 & 8 & 16 & 32 & 48\\
& pick@$V_\phi$ & 0.58 & 0.74 & 0.84 & 0.95 & 0.99\\
& Pearson & 0.45 & 0.70 & 0.82 & 0.94 & 0.97\\\midrule
Game24 (0.53), 64 steps & $k$ & 1 & 8 & 16 & 32 & 48\\
& pick@$V_\phi$ & 0.57 & 0.61 & 0.80 & 0.96 & 0.99\\
& Pearson & 0.44 & 0.60 & 0.76 & 0.94 & 0.98\\
& token AUC & 0.71 & 0.70 & 0.71 & 0.72 & 0.73\\\bottomrule\end{tabular}\end{table}

\subsubsection{Per-position critic values as a credit map}
\label{app:critic_maps}
The 2D critic produces one reading per generated position at every step of the trajectory. Figure~\ref{fig:critic_maps} shows $r_{k,i}$ for held-out rollouts of the jointly trained critics on Game24 and Sudoku. Two things are visible. First, the pooled value separates successful from failed trajectories progressively along the trajectory: the critic has learned when the trajectory commits. Second, the per-position readings carry some localisation but are dominated by the sequence-level signal: on Sudoku, ranking all cells of 96 held-out samples by $r_{k,i}$ separates wrong from right cells well, but most of that comes from wrong samples being uniformly low; restricted to the cells within incorrect samples the AUC is $0.56$--$0.63$ at every $k$. On Game24 the maps are organised by equation, but they are not a step verifier. The per-position readings serve as a ranking and guidance score (Appendix~\ref{app:loop}); the actor ascends their pooled value.

\begin{figure}[H]
\centering
\includegraphics[width=0.58\linewidth]{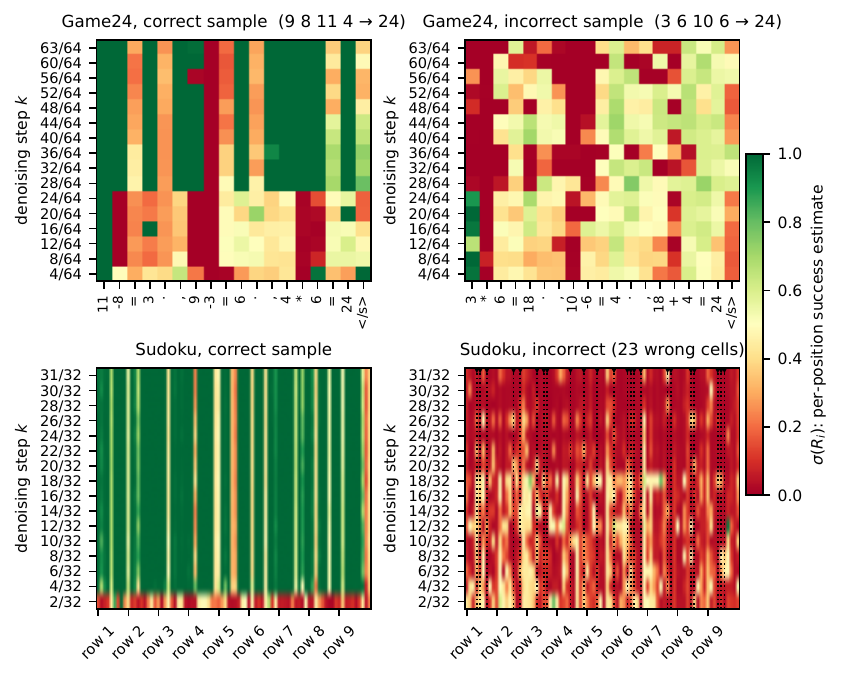}
\caption{Per-position critic readings $r_{k,i}$ (colour) along the deployment-sampler trajectory (rows: denoising step $k$; columns: generated positions, labelled by the decoded token for Game24 and by Sudoku row) for one correct and one incorrect held-out sample per task. The correct samples are the median-valued correct rollout; the incorrect Game24 sample is the median-valued failed rollout and the incorrect Sudoku sample the failed rollout with the fewest wrong cells (dotted lines and markers).}
\label{fig:critic_maps}
\end{figure}

\subsubsection{1D against 2D critic over training, all tasks}
\label{app:traj}
Figure~\ref{fig:traj} extends the left panel of Figure~\ref{fig:dyn} from Game24 to the four tasks and adds the AUC of the jointly trained head. On Game24 and K\&K both heads keep up with the policy and the 2D critic's policy climbs further; on Countdown the pooled head falls behind the moving policy and its policy turns back; on Sudoku both heads keep their AUC while both policies turn back early, the 2D critic's later and from a higher peak.

\begin{figure}[!t]
\centering
\includegraphics[width=0.85\linewidth]{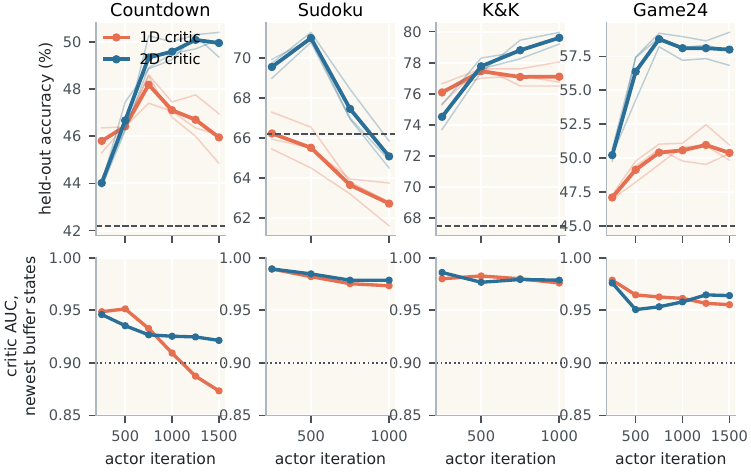}
\caption{1D against 2D critic over training: the runs of Table~\ref{tab:main}, individual runs (thin) and their mean (thick). Top: held-out accuracy (\%) at the deployment sampler of the checkpoint after that many actor iterations; dashed: the SFT base. Bottom: AUC of the jointly trained head against the outcome on the newest 256 replay-buffer states, averaged over 250-iteration windows of the training logs; dotted: the floor $a_0=0.9$ below which Algorithm~\ref{alg:jpc} gives the critic extra steps. Bottom-row curves are means over the runs. On Countdown the pooled head falls behind the moving policy and its policy turns back; on Sudoku both heads keep their AUC and the 1D critic's policy declines from the start; on Game24 and K\&K both heads keep up and the 1D critic's gain is smaller.}
\label{fig:traj}
\end{figure}

\subsubsection{The 1D and the 2D critic are interchangeable at a fixed policy}
\label{app:critic_equiv}
Table~\ref{tab:critic_equiv} compares the pre-trained 1D and 2D critics of the Table~\ref{tab:main} runs on the same states of the same policy, and their jointly trained heads on fixed probe sets. Measured this way the two critics discriminate the reward equally well, prescribe parameter updates at cosine $0.83$--$0.90$ of each other, react equally, and weakly, to a single wrong token, and their heads refit equally stably. None of the fixed-policy measurements separates the two critics: whatever distinguishes the policies they train (Figures~\ref{fig:dyn} and~\ref{fig:traj}) is not visible in either critic's value, slope or stability at a fixed policy, which is why we attribute the difference to the joint training rather than to what the critics can represent.

\begin{table}[!t]\centering\small
\caption{1D against 2D critic on the same states. Pre-trained critics unless a row says otherwise; $k$ is the denoising step of the state.}\label{tab:critic_equiv}
\footnotesize\setlength{\tabcolsep}{3pt}
\begin{tabular}{@{}p{0.35\linewidth}p{0.20\linewidth}>{\raggedright\arraybackslash}p{0.20\linewidth}>{\raggedright\arraybackslash}p{0.20\linewidth}@{}}\toprule
measurement & setting & 1D & 2D\\\midrule
Pearson correlation of the pooled value with the outcome, held-out rollouts & pre-training, Game24 $k{=}48$ / Sudoku $k{=}18$ / Countdown $k{=}14$ & 0.97 / 0.94 / 0.99 & 0.97 / 0.94 / 0.99\\
pathwise SNR of the actor's parameter gradient (Table~\ref{tab:snr} statistic) & Game24 $k{=}16$ / 24 / 32 & 0.19 / 0.17 / 0.20 & 0.18 / 0.18 / 0.15\\
half-split cosine of the actor's parameter gradient & Game24 $k{=}16$ / 24 / 32 & 0.81 / 0.78 / 0.83 & 0.81 / 0.79 / 0.70\\
cosine between the two critics' mean actor gradients & Game24 $k{=}16$ / 24 / 32 & \multicolumn{2}{l}{0.90 / 0.88 / 0.83}\\
cosine between the two critics' per-position gradient-norm profiles & Game24, real tokens & \multicolumn{2}{l}{0.94 / 0.94 / 0.94}\\
mean pairwise cosine of per-position state gradients $\partial V_\phi/\partial\rvz_{t_k}[i]$ (shared direction) & Game24 $k{=}24$ / Sudoku $k{=}12$ & 0.03 / 0.06 & 0.04 / 0.07\\
pooled $\Delta V_\phi$ (logits) for one wrong token in a successful state & Game24 $k{=}24$, re-encoded & $-0.86$ & $-0.89$\\
\quad the same, corruptions that break the answer & Game24 $k{=}24$ & $-1.76$ & $-2.04$\\
pooled $\Delta V_\phi$ for one wrong cell (never changes the decoded grid) & Sudoku $k{=}12$ & $-0.11$ & $-0.03$\\
jointly trained head: AUC on states of its own peak policy, iteration 0 $\to$ end & Game24 / Sudoku / Countdown & 0.97$\to$0.96 / 0.98$\to$0.98 / 0.90$\to$0.92 & 0.96$\to$0.96 / 0.98$\to$0.98 / 0.90$\to$0.91\\
jointly trained head: AUC on the newest buffer states, mean over training & Game24 / Sudoku / K\&K / Countdown & 0.96 / 0.98 / 0.98 / 0.92 & 0.96 / 0.98 / 0.98 / 0.93\\
actor gradient after 20 head refits: cosine to the gradient before them & Game24 $k{=}32$, buffer displaced 5\%/refit & 0.98 & 0.99\\
\bottomrule\end{tabular}\end{table}

\subsubsection{Frozen versus joint critic}
\label{app:frozen}
The right panel of Table~\ref{tab:critic_design} shows the exploit directly on Game24: the recipe of Algorithm~\ref{alg:jpc} with the pre-trained 2D critic held fixed. The pathwise update first helps, then the actor finds the critic's blind spots: the critic's value on the training batch climbs from about $3$ logits at iteration 250 to over $80$ by 1500 while the true success rate of the same batch falls from $78\%$ to $0\%$, and the held-out accuracy collapses. The identical run with the head refit after every actor step (Table~\ref{tab:main}) keeps the critic's value in the calibrated range. Increasing the frozen critic's quality (more pre-training rollouts, the 2D form) does not change this: whatever the critic is, the actor exploits it within a few hundred updates unless the critic is refreshed on the actor's own states.

\subsubsection{Training dynamics on Countdown and K\&K}
\label{app:dynamics}
Figure~\ref{fig:dynamics_cd_kk} shows the Algorithm~\ref{alg:jpc} runs with the 2D critic on the two tasks not covered by Figure~\ref{fig:dyn}. On Countdown the held-out accuracy rises for about 750 iterations and then holds while the replay-buffer success rate keeps climbing: the training batch keeps improving after the held-out set has levelled off. On K\&K the held-out accuracy rises monotonically and is still at its peak when the schedule ends; its buffer success rate and critic value rise together over the whole run.

\begin{figure}[H]
\centering
\includegraphics[width=0.85\linewidth]{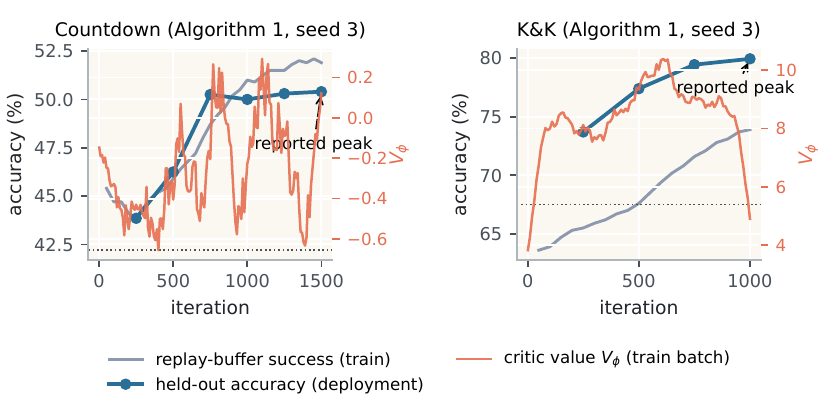}
\caption{Training dynamics of Algorithm~\ref{alg:jpc} (2D critic) on Countdown (left) and K\&K (right): held-out accuracy at the deployment sampler per checkpoint (blue), replay-buffer success rate (grey), and the critic's value on the training batch (red, right axis); dotted line: SFT base.}
\label{fig:dynamics_cd_kk}
\end{figure}

\subsection{The loop}
This group stands behind Table~\ref{tab:loop}: every round of every loop with its controls, and the critic-guided sampler that draws the later Countdown rounds.

\subsubsection{The loop in full: rounds, controls and the guided sampler}
\label{app:loop_full}
Table~\ref{tab:loop_full} gives every round behind Table~\ref{tab:loop} with the pool's coverage and the controls. On Countdown and K\&K the rounds lift the cheap sampler by 17 to 35 points from every start, RL or not. On K\&K fresh puzzles give near-full coverage from any start and the rounds close the task; RL adds nothing there that the rounds do not. On Sudoku one round covers three quarters of the pool from every start and moves accuracy by at most 1.4 points at either sampler, as the gold-solution control confirms. On Countdown two plain rounds shrink coverage and give back part of the RL gain; rounds 3 and 4, drawn with the critic-guided cheap sampler, raise coverage past its start and return the deployment sampler to the RL level, 5 to 7 points above sampling plainly from the base or the \methodGRPO{} policy. Two further plain rounds continued from the round-2 model (the two control rows after round 4) leave coverage and accuracy at the round-2 level, whereas the critic-guided rounds from the same model cover about six thousand more problems and recover $4.5$--$5.9$ points at deployment; the recovery is therefore due to the critic-guided pools, not to the additional rounds. 


\begin{table}[H]
\centering\small
\caption{The amortization loop, held-out accuracy (\%) at three samplers. Each block is one task: the SFT base, the 2D-critic policy the loop starts from, the model at the end of each round of Algorithm~\ref{alg:loop} (Countdown: all four rounds; K\&K: first and last of five; Sudoku: its one round), and, indented, the controls at their last round, which run the same rounds without a critic from the SFT base, from the \methodGRPO{} policy, and on the gold derivations of the same pool, and, for Countdown, two further plain rounds continued from the round-2 model. \emph{Covered}: pool problems (of 50k) with at least one verified sample in that round. Game24 has no pool of unlabelled problems. Sudoku's cheap sampler is $\gamma14$@16 and its deployment sampler $\gamma28$@32. Bold: the last round of each loop where it is over 1.0 point above the start.}
\label{tab:loop_full}
\setlength{\tabcolsep}{4.5pt}
\begin{tabular}{@{}llcccc@{}}
\toprule
 & pool sampler & covered & 16-step sampler & ODE-64 & deployment \\
\midrule
\multicolumn{6}{@{}l}{\emph{Countdown}}\\
SFT base & & & 23.8 & 28.5 & 42.8 \\
2D-critic policy (start) & & & 24.3 & 34.8 & 50.5 \\
round 1 & deployment sampler & 35.1k & 37.9 & 42.4 & 45.8 \\
round 2 & deployment sampler & 31.2k & 40.8 & 42.6 & 46.0 \\
round 3 & critic-guided, $\gamma4$@16 & 37.2k & \OURScdthreeg & \OURScdthreeo & \OURScdthree \\
round 4 & critic-guided, $\gamma4$@16 & 37.8k & \textbf{\OURScdg} & \textbf{\OURScdo} & \OURScd \\
\quad plain round 3 (control) & deployment sampler & 31.2k & 41.1 & 43.7 & 45.5 \\
\quad plain round 4 (control) & deployment sampler & 31.3k & 41.1 & 43.4 & 45.5 \\
\quad control: from the SFT base & deployment sampler & 31.2k & 42.7 & 44.5 & 46.4 \\
\quad control: from \methodGRPO{} & deployment sampler & 30.9k & 41.1 & 42.9 & 44.7 \\
\quad control: gold derivations & (labels) & 50.0k & 22.5 & 26.4 & 40.8 \\
\midrule
\multicolumn{6}{@{}l}{\emph{K\&K}}\\
SFT base & & & 58.6 & 57.5 & 68.2 \\
2D-critic policy (start) & & & 68.8 & 69.7 & 79.8 \\
round 1 & deployment sampler & 41.1k & 73.4 & 76.6 & 82.1 \\
round 5 & deployment sampler & 48.8k & \textbf{\OURSkkg} & \textbf{\OURSkko} & \textbf{\OURSkk} \\
\quad control: from the SFT base & deployment sampler & 49.4k & 93.5 & 97.0 & 99.6 \\
\quad control: from \methodGRPO{} & deployment sampler & 48.3k & 91.7 & 94.6 & 97.9 \\
\quad control: gold labels & (labels) & 50.0k & 94.8 & 98.0 & 99.6 \\
\midrule
\multicolumn{6}{@{}l}{\emph{Sudoku}}\\
SFT base & & & 60.0 & 60.6 & 64.0 \\
2D-critic policy (start) & & & 63.4 & 65.3 & 68.9 \\
round 1 & deployment sampler & 39.3k & 63.1 & 65.4 & \textbf{70.3} \\
\quad control: from the SFT base & deployment sampler & 36.1k & 58.7 & 60.7 & 63.6 \\
\quad control: from \methodGRPO{} & deployment sampler & 39.3k & 64.2 & 67.7 & 69.6 \\
\quad control: gold solutions & (labels) & 50.0k & 63.9 & 65.2 & 69.7 \\
\bottomrule
\end{tabular}
\end{table}

\section{Related work}
\label{app:related}
\paragraph{Actor-critic control.}
Stochastic value gradients \citep{heess2015svg} differentiate a learned value through a reparameterised policy and a transition model; DDPG and SAC \citep{lillicrap2015ddpg,haarnoja2018sac} are the model-free cases. The SVG variants of \method{} are this pattern on a generative sampler whose transition model is the sampler itself, with a value that is a success probability fitted to verifier rewards and, in \methodSVGTwoD{}, a reading per position.

\paragraph{Self-training with verifiers.}
STaR, RFT, ReST-EM and related work \citep{zelikman2022star,yuan2023rft,singh2023restem,huang2022selfimprove} fine-tune a model on its own verified generations; expert iteration \citep{anthony2017exit} alternates a stronger search-based policy with distillation. The amortization loop is this pattern with an RL policy as the sampler and, in its guided form, with the RL critic as the search heuristic; the controls in Appendix~\ref{app:loop_full} separate what the starting policy contributes from what the rounds contribute. DAV \citep{lee2026dav} casts diffusion alignment as variational expectation-maximisation; a round of Algorithm~\ref{alg:loop} can be read as the M-step of that view for one prompt, with the RL policy as the sampler.

\paragraph{Tasks.}
Countdown and Game24 follow \citet{gandhi2024sos} and \citet{yao2023tot}; Sudoku follows \citet{shah2024sudoku}; Knights-and-Knaves follows \citet{xie2024kk}.

\end{document}